\relax
\documentclass[letterpaper]{article} 
\usepackage{proceed2e}
\usepackage[margin=1in]{geometry}
\usepackage{times}  
\usepackage{helvet} 
\usepackage{courier}  
\usepackage[hyphens]{url}  
\usepackage{graphicx} 
\usepackage{graphicx}  
\usepackage[utf8]{inputenc} 
\usepackage[T1]{fontenc}    
\usepackage{url}            
\usepackage{booktabs}       
\usepackage{amsfonts}       
\usepackage{nicefrac}       
\usepackage{microtype}      
\usepackage{algorithm}
\usepackage{algorithmic}
\usepackage{amsmath}
\usepackage{amsthm}
\usepackage{mathrsfs}
\usepackage{amssymb}
\newtheorem{lem}{Lemma}
\newtheorem{thm}{Theorem}
\usepackage{graphicx}  
\usepackage{threeparttable}
\usepackage{stfloats}
\usepackage{multirow}
\usepackage{subfigure}
\usepackage{times}
\usepackage{xcolor}
\usepackage{soul}
\usepackage[utf8]{inputenc}
\usepackage[small]{caption}
\usepackage{cleveref}

\author{
  Xiaojin Zhang,\textsuperscript{\rm 1}
  Shuai Li,\textsuperscript{\rm 2}
  Weiwen Liu,\textsuperscript{\rm 1}
  Shengyu Zhang \textsuperscript{\rm 3}\\
  \textsuperscript{\rm 1}The Chinese University of Hong Kong,
  \textsuperscript{\rm 2}Shanghai Jiao Tong University,
  \textsuperscript{\rm 3}Tencent\\
  xjzhang@cse.cuhk.edu.hk,
  lishuai.sherry@gmail.com,
  wwliu@cse.cuhk.edu.hk,
  shengyuzhang@gmail.com}

\title{Contextual Combinatorial Conservative Bandits}

\begin{document}

%

%

\maketitle

\begin{abstract}
	The problem of multi-armed bandits (MAB) asks to make sequential decisions while balancing between exploitation and exploration, and have been successfully applied to a wide range of practical scenarios. Various algorithms have been designed to achieve a high reward in a long term. However, its short-term performance might be rather low, which is injurious in risk sensitive applications. Building on previous work of conservative bandits, we bring up a framework of contextual combinatorial conservative bandits. An algorithm is presented and a regret bound of $\tilde O(d^2+d\sqrt{T})$ is proven, where $d$ is the dimension of the feature vectors, and $T$ is the total number of time steps. We further provide an algorithm as well as regret analysis for the case when the conservative reward is unknown. Experiments are conducted, and the results validate the effectiveness of our algorithm.

	



\end{abstract}

\section{Introduction}

The problem of multi-armed bandits has been extensively studied and drawn lot of attention in the past decades \cite{bubeck2012regret}. In canonical stochastic multi-armed bandit problem, the learner is presented with a set of arms, whose rewards are independently and identically distributed. The learner is allowed to select one arm at each round, and the final goal is to maximize the cumulative rewards. The key challenge of the bandit problem lies in the exploitation-exploration trade-off. On the one hand, the historical decisions and observed rewards could be exploited to find the arm with the highest empirical mean reward so far; on the other hand, the learner could explore other arms to get a better estimate of their mean value, which helps to achieve a large cumulative reward in a long run. 

In some practical applications, the learning agent needs to select a \emph{set} of arms at each round, instead of just one arm. For example, a company has newly hired several employees, and a certain amount of tasks are awaited to be completed. These employees could all deal with these tasks, but the completion qualities might be distinct. Actually, the rewards gained by the employees dealing with various tasks follow some specific distributions, while the company has no prior information about the distributions. The company could adjust the assignment policy based on the observations of the employees' performance, in hope of allocating the tasks wisely to maximize the cumulative revenue. This kind of problem fits into the framework of \emph{combinatorial bandits}. Each employee-task pair could be regarded as an arm, the allocation policy is an action, and the feasible action set contains allocations which could form a matching.

Contextual bandit problem is another extended version of the traditional multi-armed bandit problem. It is supposed that the agent could observe certain contexts before making decisions at each round. Besides, the learner is aware of a hypothesis class, which could map contexts to arms and assist the learner in finding the arm with the highest reward. Indeed, the contextual bandit are sometimes referred to as partial-label, associative bandit, multi-armed bandit with expert advice and associative reinforcement learning \cite{beygelzimer2011contextual}. The stochastic contextual bandit has been widely applied in lots of areas including online advertisement selection and news recommendation \cite{chapelle2011empirical,yue2011linear,tang2013automatic,yu2016linear,wang2017efficient}, automated vaccine design and sensor management \cite{krause2011contextual}, influence maximization \cite{fang2014networked,vaswani2017diffusion}, and offline evaluator construction \cite{li2011unbiased}.

Most algorithms proposed for either contextual bandit or combinatorial bandit aim at pursuing high reward, but usually ignore the instantaneous safety guarantee. For some risk sensitive applications such as financial investment, each action should be rather cautious to avoid the possibility that the cumulative return is below certain threshold at some rounds. As for the widely used Upper Confidence Bound (UCB) strategy, it follows the principle that more explorations are in need especially when information is scarce, thus the policy is prone to explore the arms which have not been sufficiently sampled. From this point of view, this strategy is very effective for not missing any arm that might later have excellent performance. However, it is highly likely to choose some arms with low expected rewards but high variances at an early stage of the algorithm, resulting in a great amount of cumulative loss at some point of time. Similar concerns were recently addressed in \cite{wu2016conservative}, in which they introduced the conservative bandit problem and gave an algorithm  that satisfied the revenue constraint uniformly in time under traditional bandit setting. Successively, \cite{kazerouni2016conservative} considered safety control in stochastic contextual bandit and assumed that the reward of each arm is linear with its feature vector. 

In this paper, we make the following contributions. (1) We generalize the revenue constraint to a more widely used contextual combinatorial setting, obtaining a novel framework of \emph{contextual combinatorial conservative bandits}. At each round, a set of arms could be selected, with the number of arms not exceeding a specified number. The learner could only observe the rewards of the played arms, which is referred to as \emph{semi-bandit feedback}. Viewing that the reward of a set of arms is not purely in the linear combination form in most application scenarios, the reward function is generalized to non-linear case, satisfying two mild assumptions. (2) We then propose two algorithms, one for the case when the expected reward of the conservative arm set is prescribed, and one for the case when it is unknown.  Both algorithms aim at maximizing the expected reward in the long run, while simultaneously satisfying the revenue constraint for each instant round. Importantly, our algorithm is computationally more efficient as compared with \textit{CLUCB} proposed by \cite{kazerouni2016conservative} in the contextual conservative setting. (3) We provide detailed theoretical analysis for both algorithms, showing that the safety could be well maintained even in complex environments. Specifically, we consider the case when the dimension of the contextual information is rather large. With well-designed lower bound of the revenue constraint in the contextual combinatorial setting, we provide novel proofs and obtain a tighter bound than that presented in \cite{kazerouni2016conservative} in terms of high-dimensional contextual information, making it more practical for high-dimensional contextual scenarios. (4) We conduct several experiments by numerical simulation, the results are consistent with our theoretical analysis and confirm the advantage of our algorithm in terms of maintaining the cumulative reward above the specified safety threshold.


\section{Related Work}
Multi-armed bandit was initially proposed by \cite{robbins1985some}. In the traditional multi-armed bandit problem, only one arm is allowed to be selected at each round. \cite{anantharam1987asymptotically} firstly extended the setting and assumed that multiple arms could be selected. However, the number of arms that could be selected at each round is fixed and constant. \cite{gai2012combinatorial} generalized this setting, allowing any constraint on weights with regard to the selected arms, then a line of work associated with combinatorial bandit \cite{chen2013combinatorial,chen2014combinatorial,lin2014combinatorial,kveton2015tight,chen2016combinatorial} were developed successively.

For many application scenarios like Internet advertisement selection and multi-modal message generation, contextual bandit turns out to be more suitable than the traditional bandit \cite{langford2008epoch}. Different from semi-bandit feedback in which the learner could only observe the reward of the selected arm, the features of all arms are available in the contextual setting, which could assist the learner in estimating the rewards of the arms that are not selected and thus improve the performance. Consequently, a line of research considers the contextual bandit problem, typically under linear realizability assumption \cite{filippi2010parametric,li2010contextual,abbasi2011improved,chu2011contextual}, where the expectation of each arm's expected reward is a linear function of the features. To make the setting more applicable to online recommendation, \cite{qin2014contextual} developed a novel framework called contextual combinatorial bandits and proposed an effective algorithm based on the upper confidence bound strategy.

In the traditional multi-armed bandits, the upper confidence bound strategy is not only intuitively sensible but also has been proven to be efficient and asymptotically optimal \cite{lai1985asymptotically,agrawal1995sample,auer2002finite}. This strategy tends to follow the principle of \textit{optimism in face of uncertainty} and  ignores the risk it might face from a pessimistic point of view. However, safety insurance is the first priority in some areas including those related to health and finance. Recently, \cite{wu2016conservative} paid attention to safety guarantee in the process of pursuing high rewards in the traditional multi-armed bandit setting, and proposed a conservative algorithm based on the upper confidence bound strategy. The algorithm maintains the constraint that the cumulative reward must not be less than a specific percentage of the reward gained by a conservative arm uniformly over time. Subsequently, \cite{kazerouni2016conservative} considered the revenue constraint in contextual bandits, under the assumption that the reward is linear and proposed an algorithm called \textit{conservative linear UCB} (abbreviated as \textit{CLUCB}).  In order to calculate the confidence bounds of expected reward of each arm, it needs to search within a confidence set containing an infinite number of elements, while our algorithm could provide the bounds directly, leading to a great improvement in terms of the efficiency of the algorithm.

\section{Problem Formulation}
In the basic contextual combinatorial setting, the learner could select a subset of arms from a set $E = \{1,2,\dots, M\}$ of  $M$ base arms at each round, also referred to as a \textit{super arm}, subject to certain constraints on the selected arms. Suppose the maximum number of chosen items at each round is $K$, then the set of \textit{feasible actions} is the set of super arms with size less than or equal to $K$, denoted by $\Theta^K$. At round $t$, each arm $e\in E$ is associated with a feature vector $x_{t,e} \in\mathbb{R}^{d}$, and the weight $w_{t,e}$ could be represented as $w_{t,e}=\theta^{T}_{*}x_{t,e}+\epsilon_{t,e}$,
where $\theta_{*}\in\mathbb{R}^{d}$ is a fixed but unknown parameter, and $\epsilon_{t,e}$ is a random noise with zero mean. We take a standard assumption about the bound of length that $\Vert\theta_{*}\Vert_2\le\uppercase{S}$ and $\Vert{x}_{t,e}\Vert^{2}_2\le\uppercase{L}$, for all $t$ and all $e\in E$. Besides, the random noise is assumed to be conditionally 1-sub-Gaussian, i.e., for any $\gamma \in \mathbb{R}$, 
\begin{align}
\mathbb{E} [\exp(\gamma\epsilon_{t,e}) | \{A_{1:t-1}, w_{1:t-1},x_{1:t}\}] \le \exp(\gamma^{2}/2).
\end{align}
Since the expectation of the random noise is zero, the expectation of the weight $w_{t,e}$ is $\theta^{T}_{*}x_{t,e}$, represented as $w^{*}_{t,e}$.  

After choosing a super arm $A_{t}$ at round $t$, the learning agent observes the weights of the arms in $A_t$, and the expected reward of $A_{t}$ is $f(A,\mbox{\boldmath$w$}^*_t) = \tilde{f}(\mbox{\boldmath$w$}^*_t|_A)$, where $\mbox{\boldmath$w$}|_A = (w_e)_{e\in A}$ is a $|A|$-dimensional vector, $\mbox{\boldmath$w$}^{*}_{t}=(w^{*}_{t,e})_{e\in E}$, and $\tilde{f}$ is a function satisfying the following two natural properties.
\begin{itemize}
	\item \textbf{Monotonicity}\; The function $\tilde{f}(w)$ is non-decreasing with respect to $\mbox{\boldmath$w$}$, i.e. $\tilde{f}(\mbox{\boldmath$w$})\le \tilde{f}(\mbox{\boldmath$w$}')$ if $w_{e} \le w'_{e}$ for all $e$.
	\item \textbf{Lipschitz Continuity}\; For any two vectors $w$ and $w'$, we have $|\tilde{f}(\mbox{\boldmath$w$}) - \tilde{f}(\mbox{\boldmath$w$}')|\le P\|\mbox{\boldmath$w$}-\mbox{\boldmath$w$}'\|_2$. 
\end{itemize}

When it comes to the conservative setting, there is also a set of $|A_{0}| (|A_{0}|\le K)$ other \footnote{Here we assume that $E\cap A_0 = \emptyset$ mainly for the ease of conceptual understanding, but it is not hard to see that our algorithms work in the case of $E\cap A_0 \neq \emptyset$ as well---one just does not need to estimate the upper bound and lower bound of the expected reward for any action contained in both sets.} base arms $A_0 =\{M+1, \dots, M+|A_{0}|\}$, referred to as the conservative action, whose expected weight is also given by $w_{t,e}^* = \theta^{T}_{*}x_{t,e}$. In terms of the conservative perspective of the learner, the cumulative reward should not be less than a certain fraction of the reward gained by simply choosing the conservative arm set $A_0$ at each round, referred to as the \textit{revenue constraint}: 
\begin{align}
\sum\limits_{s=1}^{t}f(A_{s},\mbox{\boldmath$w$}^{*}_{s})\ge (1-\alpha)f(A_{0},\mbox{\boldmath$w$}_{s}^{*})t, \quad \forall t\in [T]. \label{constraint}
\end{align}

Here $A_{0}$ is the action selected by the conservative policy that our algorithm aims to out-perform, and its expected reward is $\mu_0 := f(A_{0},\mbox{\boldmath$w$}_{0}^{*})$. The parameter $\alpha$ determines how conservative the agent should be, where $\alpha\in(0,1)$. Denote by $A^{*}_{t}=\operatorname{argmax}_{\uppercase{A}\in\Theta^K}f(A,\mbox{\boldmath$w$}^{*}_{t})$, where $\Theta^K = A_0\cup\{A\subseteq E: |A|\le K\}$ is the set of feasible actions. We assume that $f(A,\mbox{\boldmath$w$}^{*}_{t})\ge 0$ for any action $A\in \Theta^{K}$, and there exist $\Delta_{min}$ and $\Delta_{max}$, such that $\Delta_{min}\le$ $f(A^{*}_{t},\mbox{\boldmath$w$}^{*}_{t})-\mu_0\le\Delta_{max}$ and $\alpha\mu_{0}+\Delta_{min}>0$. 



The goal of the learner is to minimize the pseudo-regret, which is defined as
\begin{align}
R(T)=\sum\limits_{t=1}^{T}\left[f(A^{*}_{t},\mbox{\boldmath$w$}^{*}_{t})-f(A_{t},\mbox{\boldmath$w$}^{*}_{t})\right].
\end{align}

\section{Algorithms}
We propose two algorithms based on the upper confidence bound strategy, to solve the contextual combinatorial bandit problem with the revenue constraint, in both cases when the conservative reward is prescribed and when it is unknown. Note that for the second case, the conservative arm set $A_0$ is prescribed while its expected reward is unknown.

\begin{algorithm}[h]
	\caption{CCConUCB with Conservative Reward}
	\label{alg:Contextual Combinatorial Conservative Bandits}
	\begin{algorithmic}[1]
		\STATE{\textbf{Input: }$\alpha,\Theta^K,\mu_0$}
		\STATE{\textbf{Initialization: }$\hat{\theta}_{0}\leftarrow 0_{d\times 1},V_{0}\leftarrow \lambda I,Y_{0} \leftarrow 0_{d\times 1}$, $N_0 = D_0 = \emptyset$, and $H_{0}\leftarrow \sqrt{\lambda}\uppercase{S}+\sqrt{\log(1/\delta^{2})}$}

		\FOR{$t\leftarrow1,2,...,T$}
		\STATE{$f(A_0,U_{t,t})\leftarrow\mu_{0}, f(A_0,L_{t,t})\leftarrow\mu_{0}$}\label{alg1_mu_0_u_b}
		\FOR{$e\leftarrow1,2,...,M$}\label{alg1_upstart}
		\STATE{$U_{t,t,e}\leftarrow\hat{\theta}^\mathrm{T}_{t-1} x_{t,e}+ H_{t-1}\Vert{x}_{t,e}\Vert_{V_{t-1}^{-1}}$}
		\STATE{$L_{t,t,e}\leftarrow \max\{0,\hat{\theta}^\mathrm{T}_{t-1} x_{t,e}- H_{t-1}\Vert{x}_{t,e}\Vert_{V_{t-1}^{-1}}\}$}
		\ENDFOR\label{alg1_upend}
		\IF { $f(A_0,U_{t,t}) > \max_{A \in \Theta^K\setminus A_0} f(A,U_{t,t})$ }
		\STATE $B_t = A_0$
		\ELSE
		\STATE $B_t = \text{arg}\max_{A \in \Theta^K\setminus A_0} f(A,U_{t,t})$
		\ENDIF
		
		\FOR{$n\in\uppercase{N}_{t-1}$ and $e\in\uppercase{A}_{n}$}
		    \STATE{$L_{t,n,e}\leftarrow \max\{0,\hat{\theta}^\mathrm{T}_{t-1}x_{n,e}-H_{t-1}\Vert{x}_{n,e}\Vert_{V_{t-1}^{-1}}\}$}\label{alg1_lowerbound}
		\ENDFOR
		\STATE{$\psi_{t}\leftarrow \sum\limits_{n\in\uppercase{N}_{t-1}}f(A_{n},L_{t,n})+f(B_{t},L_{t,t})+| D_{t-1}|\mu_{0}$}
		\IF{$\psi_{t}\ge(1-\alpha)t\mu_0$}  \label{alg1_if}
		  \STATE{$A_t\leftarrow\uppercase{B}_t,N_t\leftarrow N_{t-1}\cup \{t\}$}\label{alg1_bt}
		  \STATE{Observe $w_{t,e}$ for all $e\in\uppercase{A}_{t}$}
		  \STATE{$V_t\leftarrow\uppercase{V}_{t-1}+\sum_{e\in\uppercase{A}_{t}}{x_{t,e}x^\mathrm{T}_{t,e}}$}
		  \STATE{$Y_t\leftarrow\uppercase{Y}_{t-1}+\sum_{e\in\uppercase{A}_{t}}w_{t,e}x_{t,e}$}
		  \STATE{$\hat{\theta}_{t}\leftarrow V^{-1}_tY_t$}
		  \STATE{$H_{t}\leftarrow\sqrt{\lambda}\uppercase{S}+\sqrt{\log\left(\det(\uppercase{V}_{t})/\left(\lambda^{d}\delta^{2}\right)\right)}$}\label{update_ht}
		\ELSE \label{alg1_else}
		  \STATE{$A_t\leftarrow A_{0},t\leftarrow D_{t-1}\cup \{t\}$}\label{alg1_a0}
		\ENDIF
		\ENDFOR
	\end{algorithmic}
\end{algorithm}

\subsection{Learning with Conservative Reward Known}


The contextual combinatorial conservative bandits aimed at solving two key issues. One is to maximize the cumulative reward, the other is to guarantee the conservative constraint. For the first issue, we use the effective UCB approach, which selects the arms based on the upper confidence bound of the unknown expected weight. For the second issue, since the expected weight $w^{*}_{s}$ in the LHS of the conservative constraint Eq.\eqref{constraint} is unknown, the learner constructs a lower confidence bound $\psi_t$ of the LHS of Eq.\eqref{constraint}. If 
\begin{align}
&\psi_t \ge (1-\alpha)\mu_{0},
\end{align}
then the conservative constraint Eq.\eqref{constraint} is satisfied with high probability.



We now introduce a self-adaptive algorithm for the contextual combinatorial conservative bandit problem, assuming that enough statistics have been collected to get a good estimate of the expected reward $\mu_0$ of the conservative arm set. The pseudo-code of \textit{CCConUCB with Conservative Reward} is displayed in \textbf{Algorithm 1}, which could be divided into three main parts. First, we calculate the confidence intervals of the expected weight of each non-default arm (lines \ref{alg1_upstart}-\ref{alg1_upend}). Note that the upper and lower confidence bounds of the conservative arm set are both equal to the prescribed value $\mu_{0}$ (line \ref{alg1_mu_0_u_b}). Second, define the best action $B_t$ to be either the default action $A_0$ or the non-default best action for the upper confidence bound of the expected weights, whichever gives the larger $f$ value.  
Third, depending on whether the revenue constraint is satisfied at this moment or not (lines \ref{alg1_if} and \ref{alg1_else}), we use $B_t$ or $A_0$ (lines \ref{alg1_bt} and \ref{alg1_a0}), respectively. However, as we do not know $\mbox{\boldmath$w$}^*_s$ in Eq.\eqref{constraint}, we do not precisely know whether the constraint is satisfied. Therefore we take a conservative option here: we compute a lower confidence bound $\psi_t$ on the cumulative reward $\sum_{s=1}^t f(A_s,\mbox{\boldmath$w$}_s^*)$, and use $B_t$ if this lower confidence bound $\psi_t$ is already greater than or equal to the right hand side (RHS) of Eq.\eqref{constraint}. If the optimistic arm set is selected, we also update some statistics based on the newly received contextual information and the observed weights (lines \ref{alg1_bt}-\ref{update_ht}): the set $N_t$ (the set of rounds $s\le t$ in which we use optimistic action),  $D_t$ (the set of rounds $s\le t$ in which we use the default action), the estimate $\hat{\theta}_{t}$ to $\theta_*$, and the confidence radius $H_{t}$ of the expected weight.


\subsection{Learning with Conservative Reward Unknown}
For new applications that suffer from a cold start, neither sufficient data nor profound experiences are available to provide a reliable estimation of the performance $\mu_0$ of the conservative policy. We modify \textbf{Algorithm 1} to make it capable of handling the situation when the conservative reward is unknown, as shown in \textbf{Algorithm 2} (Appendix A). The differences between these two algorithms lie in two aspects. On the one hand, since no reliable estimation of $\mu_0$ is available, we need to calculate the confidence bounds of the expected weight of each conservative arm (lines \ref{alg2_upstart}-\ref{alg2_upend}). On the other hand, to ensure that revenue constraint Eq.\eqref{constraint} is satisfied, we use a lower bound of the LHS in Eq.\eqref{constraint} and an upper bound of the RHS in Eq.\eqref{constraint} for comparison. 

\section{Regret Analysis}
In order to maximize the cumulative reward, it is required to have a good grasp of the expected weight $w^{*}_{t,e}$ given the contextual information $x_{t,e}$. We are aware that the expected reward is linear with the contextual vector $x_{t,e}$, but the linear coefficient $\theta_{*}$ is unknown. In both algorithms, the coefficient vector is estimated using ridge regression solution $\hat{\theta}_{t}=V^{-1}_tY_t$, where $V_{t}=\lambda I+X_{t}X^\mathrm{T}_{t}$, $Y_{t}=X_{t}W_{t}$. Particularly, $X_{t}\in\mathbb{R}^{d\times\sum_{n\in\uppercase{N}_{t}}\left|A_{n}\right|}$ has columns $x_{n,e}$ where $e\in{\uppercase{A}_{n}}, n\in\uppercase{N}_{t}$, and $W_{t}\in\mathbb{R}^{\sum_{n\in\uppercase{N}_{t}}\left|A_{n}\right|}$ has rows $w_{n,e}$.

Denote by $L_{t,s,e}=\hat{\theta}^\mathrm{T}_{t-1} x_{s,e}- H_{t-1}\Vert{x_{s,e}}\Vert_{V_{t-1}^{-1}}$ and $U_{t,s,e}=\hat{\theta}^\mathrm{T}_{t-1}x_{s,e}+ H_{t-1}\Vert{x_{s,e}}\Vert_{V_{t-1}^{-1}}$. Then, according to the following lemma,  $L_{t,s,e}$ is a lower bound of the expected weight $w^{*}_{s,e}$, and $U_{t,s,e}$ is an upper bound of $w^{*}_{s,e}$, where $e\in E$. 
\begin{lem}
Assume that $\Vert\theta_{*}\Vert_2\le\uppercase{S}$ and $\Vert{x}_{t,e}\Vert^{2}_2\le\uppercase{L}$, for all t and all $e\in E$. Then, for any $\delta>0$, with probability at least $1-\delta$, we have that for all $t\ge 1$ and $s\le t$, 
\begin{align}
 \theta^\mathrm{T}_{*} x_{s,e}\ge\hat{\theta}^\mathrm{T}_{t-1} x_{s,e}- H_{t-1}\Vert{x}_{s,e}\Vert_{V_{t-1}^{-1}}
\end{align}
and
\begin{align}
 \theta^\mathrm{T}_{*} x_{s,e}\le \hat{\theta}^\mathrm{T}_{t-1} x_{s,e}+ H_{t-1}\Vert{x}_{s,e}\Vert_{V_{t-1}^{-1}},
\end{align}
where the radius $H_{t-1}$ is
\begin{align}
 H_{t-1}=\sqrt{\lambda}\uppercase{S}+\sqrt{\log\left(det(\uppercase{V}_{t-1})/\left(\lambda^{d}\delta^{2}\right)\right)}.
\end{align}
\end{lem}

The proof of \textbf{Lemma 1} is illustrated in Appendix B. Note that \textbf{Lemma 1} displays the confidence interval of the expected reward with the instantly updated variables instead of the static estimation at the specific round. Specifically, with the contextual information $x_{s,e}$ at round s, \textbf{Lemma 1} informs us the confidence interval of the expected reward based on the updated $\hat{\theta}_{t-1}$, $V_{t-1}$ and $H_{t-1}$ at any subsequent round $t$ ($t\ge s$).

In the following analysis, we use $d_{t}=\left|D_{t}\right|$ and $n_{t}=\left|N_{t}\right|$ to represent the total number of using the conservative policy and that of the optimistic policy separately.
\begin{lem}
For any $t\ge1$, $det(V_{t}) \le (\lambda + n_{t}KL/d)^d$.
\end{lem}

We proceed the proof of \textbf{Lemma 2} by taking advantage of the relationship between the \textrm{trace} and the determinant, as is shown in Appendix C.

It follows from \textbf{Lemma 2} that    
\begin{align}
H_{t}&=\sqrt{\lambda}\uppercase{S}+\sqrt{\log\left(det(\uppercase{V}_{t})/\left(\lambda^{d}\delta^{2}\right)\right)}\nonumber\\
&\le\sqrt{\lambda}\uppercase{S}+\sqrt{2\log\left(1/\delta\right)+\lowercase{d}\log{\left(1+KLt/\left(\lambda d\right)\right)}}.\nonumber
\end{align}
Denote by $C_{t}=\sqrt{2\log\left(1/\delta\right)+\lowercase{d}\log{\left(1+KLt/\left(\lambda d\right)\right)}}+\sqrt{\lambda}\uppercase{S}$,
we have $L_{t,s,e}\ge\hat{\theta}^\mathrm{T}_{t-1} x_{s,e}- C_{t-1}\Vert{x_{s,e}}\Vert_{V_{t-1}^{-1}}$, and $U_{t,s,e}\le\hat{\theta}^\mathrm{T}_{t-1}x_{s,e}+ C_{t-1}\Vert{x_{s,e}}\Vert_{V_{t-1}^{-1}}$.

The cumulative regret of choosing the optimistic policies can be bounded as presented in the following lemma. 
\begin{lem}
\noindent Assume that $\Vert\theta_{*}\Vert_2\le\uppercase{S}$ and $\Vert{x}_{t,e}\Vert^{2}_2\le\uppercase{L}$, for any $t\ge 1$ and any $e\in\uppercase{A}_t$. If $\lambda\ge L$, then for any $\delta$,with probability at least $1-\delta$ and for any $T\ge 1$, the regret bound for selecting the optimistic policy is
\begin{align}
  &\sum\limits_{t\in\uppercase{N}_{T}}\left[f(A_{t},\uppercase{U}_{t})-f(A_{t},w^{*}_{t})\right]\nonumber\\
  &\le 2PC_{T}\sqrt{2dn_{T}\log\left(1+\displaystyle\frac{n_{T}KL}{\lambda d}\right)},\nonumber
\end{align}
where $C_{T}=\sqrt{2\log\left(\displaystyle\frac{1}{\delta}\right)+\lowercase{d}\log{\left(1+\displaystyle\frac{TKL}{\lambda d}\right)}}+\sqrt{\lambda}\uppercase{S}$.\nonumber
\end{lem}
Note that we have extended the result to the case when $\Vert{x}_{t,e}\Vert^{2}_2\le\uppercase{L}$ and at most $K$ arms are allowed to be selected at each round. When $n_{T}=T$, the regret bound corresponds to that of the standard contextual combinatorial bandits, and the order of which is $O(d\sqrt{T}\max\{1,\log(TK/d)\})$. The detailed proof procedure is presented in Appendix D.

The following two lemmas present two upper bounds on $\sum\limits_{n\in\uppercase{N}_{t}}\sum\limits_{e\in\uppercase{A}_{n}}\Vert{x}_{n,e}\Vert^{2}_{V_{t}^{-1}}$. The first bound is smaller when $n_{t}K\le d$, while the second bound is smaller when $n_{t}K>d$. 

\begin{lem}
For any $t\ge 1$, suppose that $\Vert{x}_{t,e}\Vert^{2}_2\le\uppercase{L}$, then
\begin{align}
&\sum\limits_{n\in\uppercase{N}_{t}}\sum\limits_{e\in\uppercase{A}_{n}}\Vert{x}_{n,e}\Vert^{2}_{V_{t}^{-1}}\nonumber\\
  &\le n_{t}K\left[1-\left(\displaystyle\frac{\lambda}{\lambda+n_{t}KL/d}\right)^{d/\left(n_{t}K\right)}\right].\nonumber
\end{align}
\end{lem}
\begin{proof}
By the definition of $V_{t}$, it can be represented as $V_{t}=\lambda I+X_{t}X^\mathrm{T}_{t}$, where $X_{t}\in\mathbb{R}^{d\times\sum_{n\in\uppercase{N}_{t}}\left|A_{n}\right|}$,whose columns are $x_{n,a^k_{n}},k\in{\left[\left|\uppercase{A}_{n}\right|\right]}, n\in\uppercase{N}_{t}$.
Thus,
\begin{align}
  \det({\lambda I})&=\det({V_{t}-X_{t}X^\mathrm{T}_{t}})\nonumber\\
  &=\det\left(V_t\left(I-V^{-1}_{t}X_{t}X^\mathrm{T}_{t}\right)\right)\nonumber\\
  &=\det({V_t})\cdot\det({I-V^{-1}_{t}X_{t}X^\mathrm{T}_{t}})\nonumber\\
  &=\det({V_{t}})\cdot\det({I-X^\mathrm{T}_{t}V^{-1}_{t}X_{t}}),\label{eq:bb}
\end{align}
where the last equality follows from the fact that $\det({I-AB})=\det({I-BA})$.

\noindent Note that the diagonal elements of $X^\mathrm{T}_{t}V^{-1}_{t}X_{t}$ are $\Vert{x}_{n,a^k_{n}}\Vert^2_{V_{t}^{-1}},k\in{\left[\left|\uppercase{A}_{n}\right|\right]}, n\in\uppercase{N}_{t}$. Thus, we have
\begin{align}
  \textrm{trace}\left({I-X^{T}_{t}V^{-1}_{t}X_{t}}\right)&=\sum_{n\in\uppercase{N}_{t}}\left|A_{n}\right|\nonumber\\
  &-\sum\limits_{n\in\uppercase{N}_{t}}\sum\limits_{e\in\uppercase{A}_{n}}\Vert{x}_{n,e}\Vert^{2}_{V_{t}^{-1}}.\label{eq:cc}
\end{align}

\noindent Denote by $N=\sum_{n\in\uppercase{N}_{t}}\left|A_{n}\right|$, then $I-X^\mathrm{T}_{t}V^{-1}_{t}X_{t}\in\mathbb{R}^{N\times{N}}$. Let $\lambda_{1},\lambda_{2}\dots,\lambda_{N}$ be the eigenvalues of $I-X^\mathrm{T}_{t}V^{-1}_{t}X_{t}$, then
\begin{align}
  \det(I-X^\mathrm{T}_{t}V^{-1}_{t}X_{t})&=\lambda_{1}\times\lambda_{2}\times\dots\lambda_{N}\nonumber\\
               &\le\left(\left(\lambda_{1}+\lambda_{2}+\dots+\lambda_{N}\right)/N\right)^{N}\nonumber\\
               &=\left(\textrm{trace}\left(I-X^{T}_{t}V^{-1}_{t}X_{t}\right)/N\right)^N.\label{eq:dd}
\end{align}

\noindent By  \eqref{eq:bb}, \eqref{eq:cc} and \eqref{eq:dd}, we obtain
\begin{align}
  \sum\limits_{n\in\uppercase{N}_{t}}\sum\limits_{e\in\uppercase{A}_{n}}\Vert{x}_{n,e}\Vert^{2}_{V_{t}^{-1}}
  \le N\left[1-\left(\det(\lambda I)/\det(V_{t})\right)^{{1}/{N}}\right]\nonumber,
\end{align}
combined with \textbf{Lemma 2} gives
\begin{align}
\det(V_{t})\le\left(\lambda+n_{t}KL/d\right)^{d}\nonumber.
\end{align}

\noindent Thus,
\begin{align}
  \sum\limits_{n\in\uppercase{N}_{t}}\sum\limits_{e\in\uppercase{A}_{n}}\Vert{x}_{n,e}\Vert^{2}_{V_{t}^{-1}}
  &\le N\left[1-(\displaystyle\frac{\lambda}{\lambda+n_{t}KL/d})^{d/N}\right]\nonumber.
\end{align}

\noindent Denote by $a=\lambda/(\lambda+n_{t}KL/d)$, $f(x)=x(1-a^{d/x})$, then
\begin{align}
&f^{'}(x)=1-a^{d/x}\left(1+d/x\ln(1/a)\right).
\end{align}

\noindent Since $u>\ln(u)+1$ for $u>1$, we have $f^{'}(x)>0$. Thus
\begin{align}
  &\sum\limits_{n\in\uppercase{N}_{t}}\sum\limits_{e\in\uppercase{A}_{n}}\Vert{x}_{n,e}\Vert^{2}_{V_{t}^{-1}}\nonumber\\
  &\le n_{t}K\left[1-(\displaystyle\frac{\lambda}{\lambda+n_{t}KL/d})^{d/\left(n_{t}K\right)}\right].\nonumber
\end{align}
\end{proof}

\noindent Next we present a tighter bound for $\sum\limits_{n\in\uppercase{N}_{t}}\sum\limits_{e\in\uppercase{A}_{n}}\Vert{x}_{n,e}\Vert^{2}_{V_{t}^{-1}}$ when $n_{t}K>d$, as stated in \textbf{Lemma 5}.

\begin{lem}
For any $t\ge 1$, suppose that $\Vert{x}_{t,e}\Vert^{2}_2\le\uppercase{L}$, then
\begin{align}
   \sum\limits_{n\in\uppercase{N}_{t}}\sum\limits_{e\in\uppercase{A}_{n}}\Vert{x}_{n,e}\Vert^{2}_{V_{t}^{-1}}
  &\le\displaystyle\frac{n_{t}\uppercase{K}\uppercase{L}d}{\lambda d+n_{t}\uppercase{K}\uppercase{L}}.
\end{align}
\end{lem}
\begin{proof}
\begin{align}
&\sum\limits_{n\in\uppercase{N}_{t}}\sum\limits_{e\in\uppercase{A}_{n}}\Vert{x}_{n,e}\Vert^{2}_{V_{t}^{-1}}\nonumber\\
 &=\sum\limits_{n\in\uppercase{N}_{t}}\sum\limits_{e\in\uppercase{A}_{n}}\Vert{x}_{n,e}^{T}\uppercase{V}^{-1/2}_{t}\Vert_{2}^{2}\\
 &=\textrm{trace}\left[\sum\limits_{n\in\uppercase{N}_{t}}\sum\limits_{e\in\uppercase{A}_{n}}{\uppercase{V}^{-1/2}_{t}x_{n,e}x^\mathrm{T}_{n,e}\uppercase{V}^{-1/2}_{t}}\right]\\
 &=\textrm{trace}\left(\uppercase{V}^{-1/2}_{t}\left(\uppercase{V}_{t}-\lambda\uppercase{I}\right)\uppercase{V}^{-1/2}_{t}\right)\label{eq:lem41}\\
 &=\textrm{trace}\left(\uppercase{I}-\lambda\uppercase{V}^{-1}_{t}\right),
\end{align}
where \eqref{eq:lem41} is due to the definition of $V_{t}$ as $V_{t}=\lambda\uppercase{I}+\sum\limits_{n\in\uppercase{N}_{t}}\sum\limits_{e\in\uppercase{A}_{n}}{x_{n,e}x^\mathrm{T}_{n,e}}$. Let $\lambda_{1},\lambda_{2}\dots,\lambda_{d}$ be the eigenvalues of $\uppercase{V}_{t}$, then $1/\lambda_{1},1/\lambda_{2}\dots,1/\lambda_{d}$ are the eigenvalues of $\uppercase{V}^{-1}_{t}$, and
\begin{align}
\textrm{trace}(\uppercase{V}^{-1}_{t})&=\displaystyle\frac{1}{\lambda_{1}}+\dots+\displaystyle\frac{1}{\lambda_{d}}\nonumber\\
&\ge \displaystyle\frac{d^{2}}{\lambda_{1}+\dots+\lambda_{d}}\nonumber\\
&\ge \displaystyle\frac{d^{2}}{d\lambda+n_{t}\uppercase{K}\uppercase{L}}.
\end{align}
Thus,
\begin{align}
\sum\limits_{n\in\uppercase{N}_{t}}\sum\limits_{e\in\uppercase{A}_{n}}\Vert{x}_{n,e}\Vert^{2}_{V_{t}^{-1}}&= d-\lambda \textrm{trace}\left(\uppercase{V}^{-1}_{t}\right)\nonumber\\
&\le \displaystyle\frac{n_{t}\uppercase{K}\uppercase{L}d}{\lambda d+n_{t}\uppercase{K}\uppercase{L}}.
\end{align}
\end{proof}
\noindent In terms of the construction of the lower bound of the expected weight $w_{n,e}^{*}$ (line \ref{alg1_lowerbound} in Algorithm 1 and line \ref{alg2_lowerbound} in Algorithm 2), if we use $L(n,n)$ instead of $L(t,n)$, then the upper bound of $d_t$ would depend on $\sum\limits_{n\in\uppercase{N}_{t}}\sum\limits_{e\in\uppercase{A}_{n}}\Vert{x}_{n,e}\Vert^{2}_{V_{n}^{-1}}$ instead of $\sum\limits_{n\in\uppercase{N}_{t}}\sum\limits_{e\in\uppercase{A}_{n}}\Vert{x}_{n,e}\Vert^{2}_{V_{t}^{-1}}$.  Importantly, the newly proposed upper bound of  $\sum\limits_{n\in\uppercase{N}_{t}}\sum\limits_{e\in\uppercase{A}_{n}}\Vert{x}_{n,e}\Vert^{2}_{V_{t}^{-1}}$ is tighter than that of $\sum\limits_{n\in\uppercase{N}_{t}}\sum\limits_{e\in\uppercase{A}_{n}}\Vert{x}_{n,e}\Vert^{2}_{V_{n}^{-1}}$ proposed in \cite{qin2014contextual}.

\begin{lem}
For any $t\in\uppercase{D}_{T}$, we have
\begin{align}
  d_{t}<
  &\left(\left[1-(1+n_{t})\alpha\right]\mu_0-n_{t}\Delta_{min}\right)/(\alpha\mu_0)\nonumber\\
  &+2PC_{t}\sqrt{n_{t}\sum\limits_{n\in\uppercase{N}_{t}}\sum\limits_{e\in\uppercase{A}_{n}}{\Vert{x}_{n,e}\Vert^{2}_{V_{t}^{-1}}}}/(\alpha\mu_0),
\end{align}
where $C_{t}=\sqrt{\lambda}\uppercase{S}+\sqrt{2\log\left(1/\delta\right)+\lowercase{d}\log{\left(1+\displaystyle\frac{KLt}{\lambda d}\right)}}$.
\end{lem}
\textbf{Lemma 6} provides an upper bound on the total time steps the default arm set is selected, and the detailed proof is illustrated in Appendix E.

It could be inferred from \textbf{Lemma 1} that \textbf{Algorithm 1} satisfies the conservative constraint Eq.\eqref{constraint} for all $T\ge 1$ with probability at least $1 - \delta$. Similarly, \textbf{Algorithm 2} also ensures that the constraint holds for all $T\ge 1$ with probability at least $1-\delta$. 
The following theorems illustrate the regret bounds for the proposed algorithms.\\

\begin{thm}
\noindent If $\lambda\ge L$, then the following regret bound is satisfied with probability at least $1 - \delta$
\begin{align}
 &R(T) = O\left(d(d+\sqrt{T})\max\{1,\log(\displaystyle\frac{TK}{d})\}+\displaystyle\frac{d}{K}\right).
\end{align}
Furthermore, for $TK\le d$, we have
\begin{align}
  &R(T) = O\left(d\sqrt T+d^{3/2}/\sqrt{K}\right).
\end{align}
\end{thm}

Compared with the regret bound of order $ O(d\sqrt{T}\max\{1,\log(TK/d)\})$ for the contextual combinatorial bandits as proposed in \textbf{Lemma 3}, the upper bound of \text{CCConUCB} is added with an extra order $O(d^{2}\max\{1,\log(TK/d)\})$, which is caused by the requirement of satisfying the revenue constraint \eqref{constraint} at each round. The detailed proof of  \textbf{Theorem 1} is shown in Appendix F.


\begin{thm}
\noindent If $\lambda\ge L$, then the following regret bound is satisfied with probability at least $1 - \delta$
\begin{align}
  R(T)=&\, O(d\sqrt{T} \max\{1,\log(TK/d)\}\nonumber\\
  &+\max\{d^{2}\max\{1,\log(TK/d)\}, \nonumber\\
  &\sqrt{dK\max\{1,\log(TK/d)\}}\}),
\end{align}
where $C_{T}=\sqrt{2\log\left(\displaystyle\frac{1}{\delta}\right)+\lowercase{d}\log{\left(1+\displaystyle\frac{TKL}{\lambda d}\right)}}+\sqrt{\lambda}\uppercase{S}$.
Furthermore, if $TK\le d$, then
\begin{align}
  R(T)=&
  O(d\sqrt T+d^{3/2}/\sqrt{K}+\sqrt{Kd}).
\end{align}
\end{thm}

The proof of \textbf{Theorem 2} is amenable to the technique used for the analysis of the \textit{CCConUCB} when the conservative reward is available and is displayed in Appendix G.

\section{Numerical Simulations}

In this section, we conduct three simulation experiments. We compare the performance of distinct algorithms in various cases, record the frequency the non-conservative algorithm disobeys the conservative constraint, and analyze the alterations of the conservative learner's preference between the conservative policy and the optimistic policy, and compare our conservative algorithm with a naive algorithm that always follows the conservative strategy.

\begin{table*}[tp]
  \centering
  \begin{threeparttable}
  \caption{The compliance of the revenue constraints of \textit{Contextual Combinatorial UCB} and the alterations of the selection preference of \textit{CCConUCB} against distinct $\alpha$.}
  \label{tab:statistics}
    \begin{tabular}{ccccc}
    \toprule
    \hline
    \multirow{2}{*}{$\alpha$ }&
    \multicolumn{2}{c}{ \textit{Contextual Combinatorial UCB}}&\multicolumn{2}{c}{ \textit{CCConUCB} } \cr
    \cmidrule(lr){2-3} \cmidrule(lr){4-5}
    & $\#$violated constraints & $\#$satisfied constraints & $\#$optimal policies & $\#$conservative policies  \cr
    \midrule
  $0.01$ & $49,993$ & $7$ & $519$ & $49,481$\cr
  $0.15$ & $18,502$ & $31,498$ & $46,698$ & $3302$\cr
  $0.3$ & $10,666$ & $39,334$ & $49,332$ & $668$\cr
  $0.6$ & $2642$ & $47,358$ & $49,752$ & $248$\cr
  $0.9$ & $1032$ & $48,968$ & $49,946$ & $54$\cr
     \hline
    \bottomrule
    \end{tabular}
    \end{threeparttable}
\end{table*}

\begin{figure}[h]
\centering
  \includegraphics[width=\linewidth, height=0.75\linewidth]{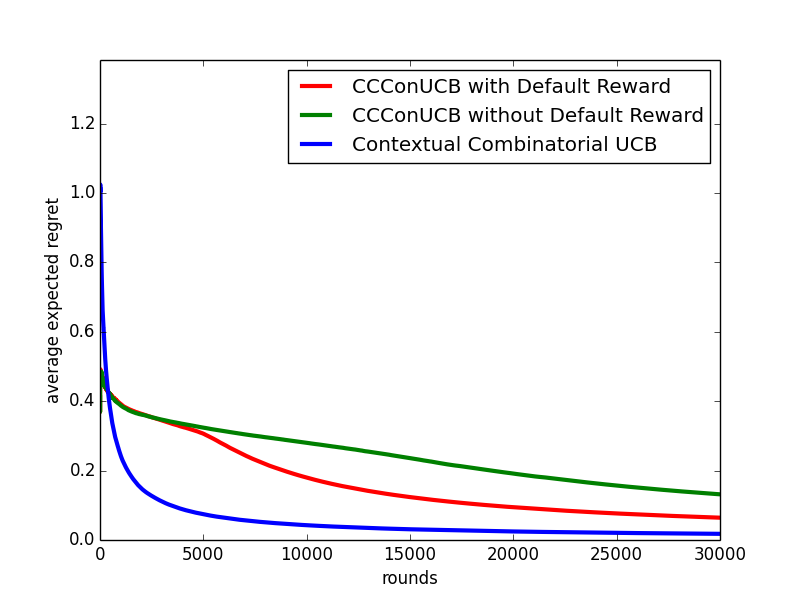}
  \caption{Average expected regret of distinct models}\label{fig:Three_Settings_Comparison}
\end{figure}

\begin{figure}[h]
\centering
  \includegraphics[width=\linewidth,height=0.75\linewidth]{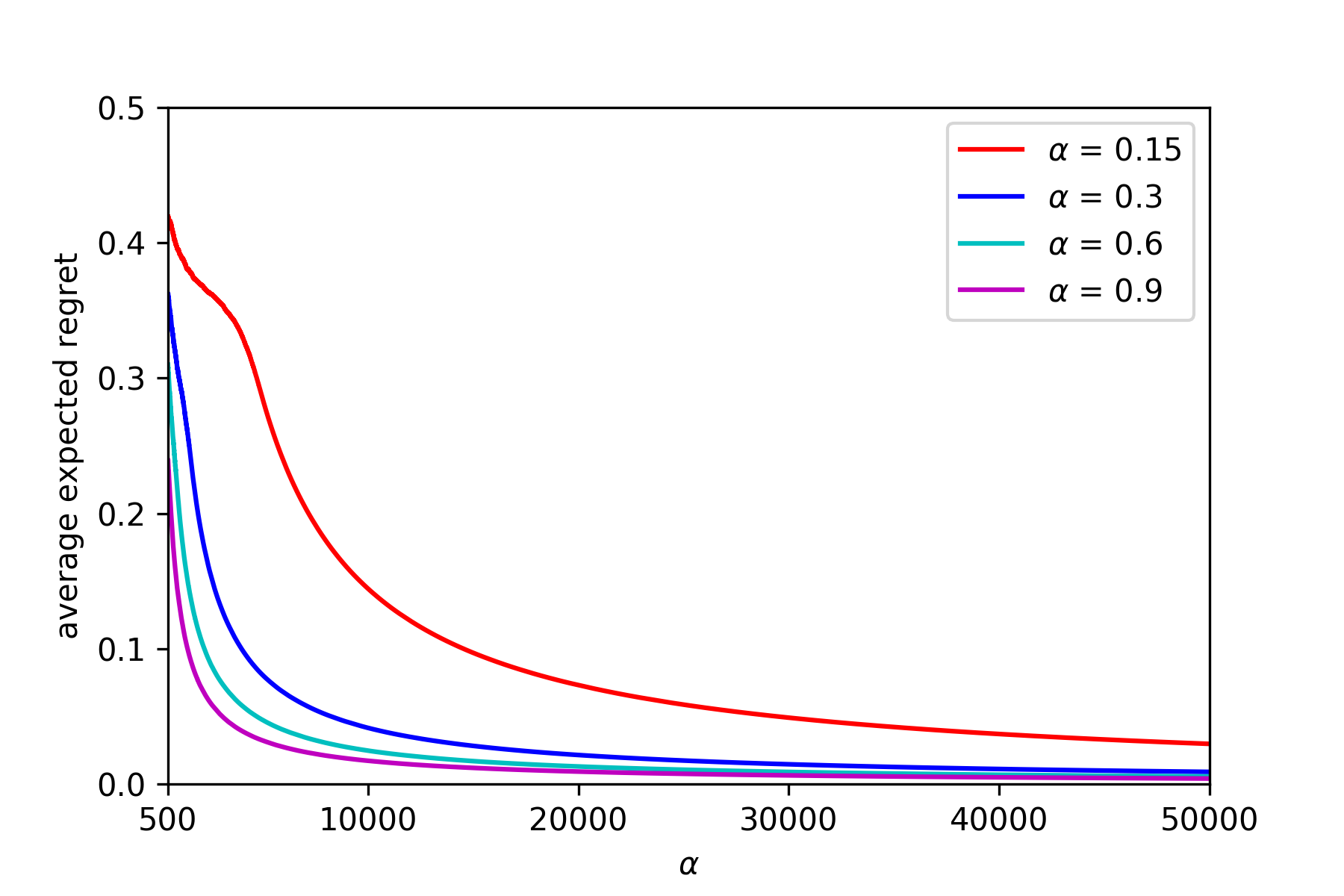}
  \caption{Average expected regret of \textit{CCConUCB}}\label{fig:CCC_WithReward_Distinct_alpha}
\end{figure}

In the first experiment, we simulated three UCB-based algorithms, which are the \textit{CCConUCB with Conservative Reward}, \textit{CCConUCB without Conservative Reward}, and \textit{$C^{2}$UCB} (\textit{Contextual Combinatorial UCB}) proposed in \cite{qin2014contextual}. The result is displayed in Figure~\ref{fig:Three_Settings_Comparison}. A total of $100$ arms are presented, each arm is associated with a 10-dimensional feature vector, and at most $2$ arms are allowed to be selected at each round. For each arm, the elements of the feature vector are randomly sampled from the uniform distribution $\mathcal{U}(-1, 1)$, satisfying that the expected reward of each arm is non-negative. The coefficient $\theta_{*}$ is randomly sampled from the normal distribution $\mathcal{N}(0,I_{10})$. The conservative coefficient is $\alpha=0.2$ for \textit{CCConUCB} in both cases. 
For the conservative arms, the feature vector of each conservative arm is randomly sampled from the uniform distribution $\mathcal{U}(-1, 1)$, satisfying that the expected reward of each arm is no less that the ninth best reward and can be up to the eighth best reward. The random noise is randomly sampled from the normal distribution $\mathcal{N}(0,1)$. As can be seen from the comparison between \textit{CCConUCB} and \textit{Contextual Combinatorial UCB}, a conservative learner might give up some opportunities for exploring more arms in order to maintain safety at each round, at the expense of a higher regret in the long run. In terms of the contextual combinatorial conservative setting, the algorithm with the information of the default reward performs better in terms of the average regret.

In the second simulation experiment, we record the total time steps \textit{Contextual Combinatorial UCB} disobeys the conservative rule in Table 1. The settings of the number of base arms, the maximum number of selected arms and the dimension of the contextual vector are identical with those in the first experiment. The smaller $\alpha$ is, the more stringent the constraint would be, leading to more violations with regard to the non-conservative algorithm. We further analyze the selection preference of \textit{CCConUCB} between optimistic policy and conservative policy for distinct conservative coefficients. Viewing from the last two columns of Table \ref{tab:statistics}, a more conservative learner tends to have a greater appetite for the conservative policy. With fewer chances for exploration, the agent learns slower about the non-conservative arms, resulting in a lower degrading speed in terms of the average expected regret, as shown in Figure~\ref{fig:CCC_WithReward_Distinct_alpha}. Consequently, although the regret of the conservative algorithm degrades at a slower speed with smaller $\alpha$, it successfully avoids a large number of constraint violations.

In the third simulation experiment, we compared the regret of the \textit{CCConUCB} with a conservative policy, where the conservative arm set is always selected at each round. We record the endurance time the \textit{CCConUCB} needs to outperform the conservative policy, the result is shown in Figure \ref{fig:endure}. Since the gap between the optimal action and the conservative action is a random sample, we divide them into three groups and take average over each group, namely low, medium and high. The result indicates that the smaller the gap is, the longer it takes for our algorithm to outperform the conservative policy. Note that the cumulative reward of \textit{CCConUCB} is not less than (1-$\alpha$) faction of that of the conservative policy, and the endurance time depends also on $\alpha$. The larger $\alpha$ is, the less stringent the requirement is. With more chances to explore, the algorithm converges at a faster speed, then it would be much easier for the \textit{CCConUCB} algorithm to outperform the pure conservative algorithm. 
\begin{figure}[h]
\centering
\includegraphics[width=\linewidth,height=0.7\linewidth]{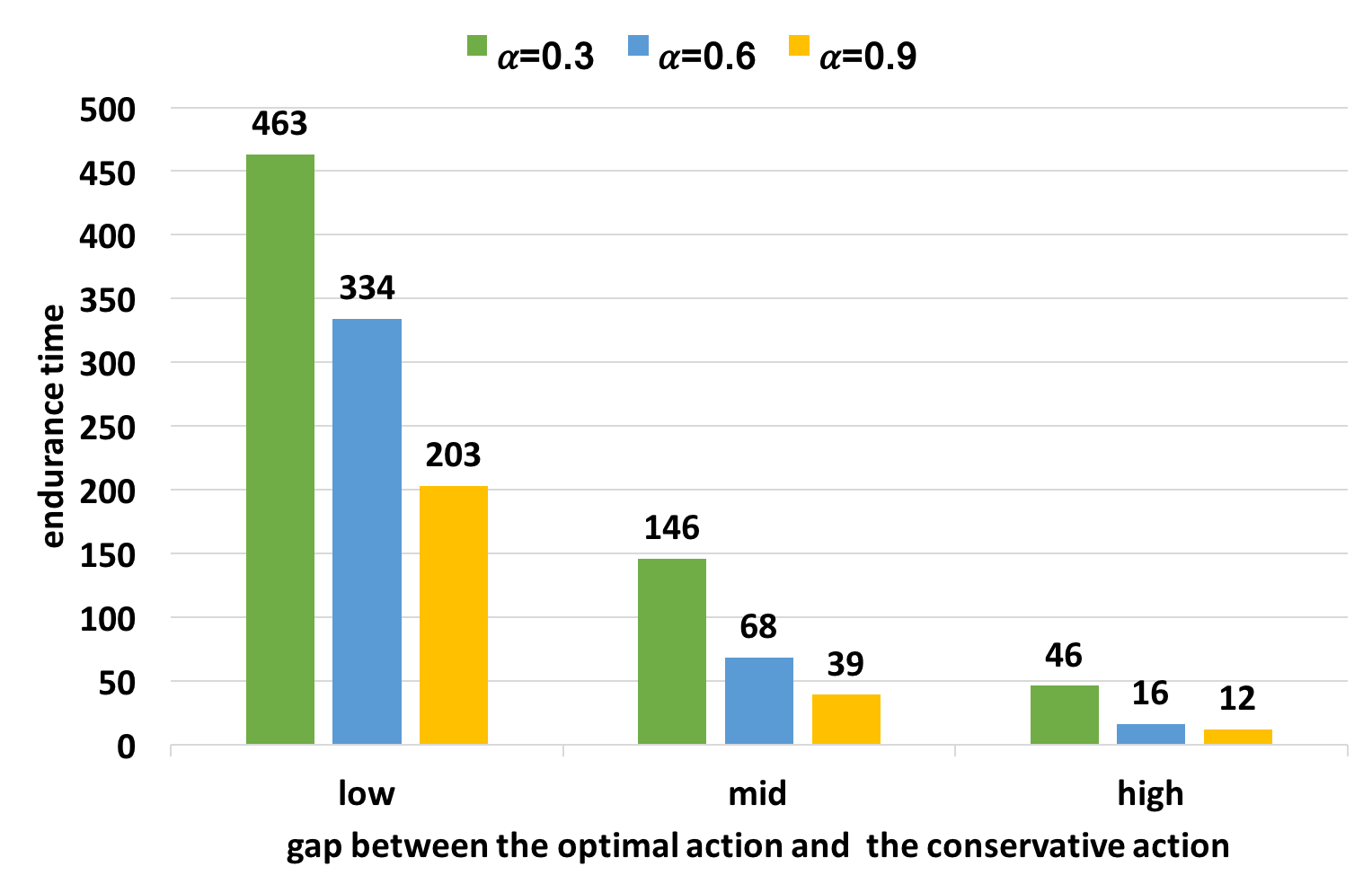} 
\caption{Endurance time to outperform the conservative policy}
\label{fig:endure}
\end{figure}


\section{Discussion and Conclusion}
We introduce a novel framework of multi-armed bandits, referred to as contextual combinatorial conservative bandits. We propose two UCB-based algorithms dealing with two cases when the conservative reward is prescribed and unknown separately, which not only maximize the expected reward in the long run, but also ensure safety at any instantaneous round. The reward $f(A,w)$ of the arm set $A$ could be as simple as a summation of the reward of each arm $a \in A$, and it can also be more complicated such as a sigmoid function. Besides, we remark that a tighter bound for high-dimensional contextual scenario is provided in the above theorem. Specifically, when reduced to the conservative contextual setting ($K=1$), and assume that $T \le d$, \textbf{Algorithm 1} achieves a regret of $O(d\sqrt{T}+d^{3/2})$, which is tighter than the order $O(d\sqrt{T}\log(T)+(d\log d)^{2})$ obtained by \cite{kazerouni2016conservative}. The baseline payoff could be time-variant or stationary, depending on the conservative learner and the specific application scenario. Our work focuses on the stationary case, i.e. the baseline payoff of the conservative learner does not change with time, while the time-variant case could be a considerable direction for future work.



\bibliographystyle{plain}
\bibliography{references}

\newtheorem *{appenlem0}{Lemma 1}
\newtheorem *{appenlemA1}{Lemma 1.1}
\newtheorem*{appenlem}{Lemma 2}
\newtheorem*{appenlem1}{Lemma 3}
\newtheorem*{appenlem2}{Lemma 6}
\newtheorem*{appenlem4}{Lemma 4}
\newtheorem*{appenlem5}{Lemma 5}
\newtheorem*{appenthm}{Theorem 2}
\newtheorem*{appenthm1}{Theorem 1}

\clearpage
\newpage

\begin{appendix}

\section{\leftline{Algorithm 2}}
\begin{algorithm}[h]
\caption{CCConUCB without Conservative Reward}
\label{alg:Contextual Combinatorial Conservative Bandits}
\begin{algorithmic}[1]
\STATE{\textbf{Input: }$\alpha,\Theta^K$}
\STATE{\textbf{Initialization: }$\hat{\theta}_{0}\leftarrow 0_{d\times 1},V_{0}\leftarrow \lambda I, Y_{0} \leftarrow 0_{d\times 1}$, $N_0 = D_0 = \emptyset$, and $H_{0}\leftarrow \sqrt{\lambda}\uppercase{S}+\sqrt{\log(1/\delta^{2})}$} 
\FOR{$t\leftarrow1,2,...,T$}
\FOR{$e\in\uppercase{A}_{0}$}\label{alg2_upstart}
\STATE{$U_{t,t,e}\leftarrow\hat{\theta}^\mathrm{T}_{t-1} x_{0,e}+ H_{t-1}\Vert{x}_{0,e}\Vert_{V_{t-1}^{-1}}$}
\STATE{$L_{t,t,e}\leftarrow \max\{0, \hat{\theta}^\mathrm{T}_{t-1} x_{0,e}-H_{t-1}\Vert{x}_{0,e}\Vert_{V_{t-1}^{-1}}\}$}
\ENDFOR\label{alg2_upend}
\FOR{$e\leftarrow1,2,...,M$}
\STATE{$U_{t,t,e}\leftarrow\hat{\theta}^\mathrm{T}_{t-1} x_{t,e}+ H_{t-1}\Vert{x}_{t,e}\Vert_{V_{t-1}^{-1}}$}
\STATE{$L_{t,t,e}\leftarrow \max\{0,\hat{\theta}^\mathrm{T}_{t-1} x_{t,e}- H_{t-1}\Vert{x}_{t,e}\Vert_{V_{t-1}^{-1}}\}$}
\ENDFOR
		\IF { $f(A_0,U_{t,t}) > \max_{A \in \Theta^K\setminus A_0} f(A,U_{t,t})$ }
		\STATE $B_t = A_0$
		\ELSE
		\STATE $B_t = \text{arg}\max_{A \in \Theta^K\setminus A_0} f(A,U_{t,t})$
		\ENDIF
\FOR{$n\in\uppercase{N}_{t-1}$ and $e\in\uppercase{A}_{n}$}
\STATE{$L_{t,n,e}\leftarrow \max\{0,\hat{\theta}^\mathrm{T}_{t-1}x_{n,e}-H_{t-1}\Vert{x}_{n,e}\Vert_{V_{t-1}^{-1}}\}$}\label{alg2_lowerbound}
\ENDFOR
\IF{$\sum\limits_{n\in\uppercase{N}_{t-1}}f(A_{n},L_{t,n})+f(B_{t},L_{t,t})+|D_{t-1}|\cdot f(A_{0},U_{t,0})\ge(1-\alpha)t f(A_{0},U_{t,0})$}
\STATE{$A_t\leftarrow\uppercase{B}_t,N_t\leftarrow N_{t-1}\cup \{t\}$}
\STATE{Observe $w_{t,e}$ for all $e\in\uppercase{A}_{t}$}
\STATE{$V_t\leftarrow\uppercase{V}_{t-1}+\sum_{e\in\uppercase{A}_{t}}{x_{t,e}x^\mathrm{T}_{t,e}}$}
\STATE{$Y_t\leftarrow\uppercase{Y}_{t-1}+\sum_{e\in\uppercase{A}_{t}}w_{t,e}x_{t,e}$}
\STATE{$\hat{\theta}_{t}\leftarrow V^{-1}_tY_t$}
\STATE{$H_{t}\leftarrow\sqrt{\lambda}\uppercase{S}+\sqrt{\log\left(\det(\uppercase{V}_{t})/\left(\lambda^{d}\delta^{2}\right)\right)}$}
\ELSE
\STATE{$A_t\leftarrow A_0,D_t\leftarrow D_{t-1}\cup \{t\}$}
\ENDIF
\ENDFOR
\end{algorithmic}
\end{algorithm}

\section{\leftline{Proof of Lemma 1}}
\begin{appenlem0}
Assume that $\Vert\theta_{*}\Vert_2\le\uppercase{S}$ and $\Vert{x}_{t,e}\Vert^{2}_2\le\uppercase{L}$, for all t and all $e\in E$. Then, for any $\delta>0$, with probability at least $1-\delta$ and for all $t\ge 1$ and $s\le t$, we have
\begin{align}
 \theta^\mathrm{T}_{*} x_{s,e}\ge\hat{\theta}^\mathrm{T}_{t-1} x_{s,e}- H_{t-1}\Vert{x}_{s,e}\Vert_{V_{t-1}^{-1}}\nonumber
\end{align}
and
\begin{align}
 \theta^\mathrm{T}_{*} x_{s,e}\le \hat{\theta}^\mathrm{T}_{t-1} x_{s,e}+ H_{t-1}\Vert{x}_{s,e}\Vert_{V_{t-1}^{-1}},\nonumber
\end{align}
where
\begin{align}
  H_{t-1}=\sqrt{\lambda}\uppercase{S}+\sqrt{\log\left(\displaystyle\frac{det(\uppercase{V}_{t-1})}{\lambda^{d}\delta^{2}}\right)}.
\end{align}
\end{appenlem0}
\begin{proof}
\begin{align}
   \left|(\theta_{*}-\hat{\theta}_{t-1})^\mathrm{T}x_{s,e}\right|&=\left|(\theta_{*}-\hat{\theta}_{t-1})^\mathrm{T}V_{t-1}^{1/2}V_{t-1}^{-1/2}x_{s,e}\right|\nonumber\\
   &\le\Vert(\theta_{*}-\hat{\theta}_{t-1})^\mathrm{T}V_{t-1}^{1/2}\Vert _{2}\cdot\Vert x_{s,e}^\mathrm{T}V_{t-1}^{-1/2}\Vert_{2}\nonumber\\
   &=\Vert(\theta_{*}-\hat{\theta}_{t-1})^\mathrm{T}\Vert_{V_{t-1}} \cdot\Vert x_{s,e}^\mathrm{T}\Vert_{V_{t-1}^{-1}}\nonumber\\
   &\le H_{t-1}\Vert{x}_{s,e}\Vert_{V_{t-1}^{-1}}\nonumber,
\end{align}
where the second inequality follows from \textbf{Theorem 2} in \cite{abbasi2011improved} (For the convenience of the reader, also recalled as Lemma 1.1 in the Appendix).
\end{proof}

\begin{appenlemA1}
\cite{abbasi2011improved} Assume that $\Vert\theta_{*}\Vert_{2}\le\uppercase{S}$, then for any $\delta>0$, with probability at least $1-\delta$, for all $t\ge 0$, we have
\begin{align}
\Vert(\theta_{*}-\hat{\theta}_{t})^\mathrm{T}\Vert_{V_{t}}\le \sqrt{\lambda}\uppercase{S}+\sqrt{\log\left(\displaystyle\frac{det(\uppercase{V}_{t})}{\lambda^{d}\delta^{2}}\right)}.\nonumber
\end{align}
\end{appenlemA1}

\section{\leftline{Proof of Lemma 2}}
\begin{appenlem}
For any $t\ge1$, $\det(V_{t}) \le (\lambda + n_{t}KL/d)^d$.
\end{appenlem}
\begin{proof}
Denote the eigenvalues of $\uppercase{V}_{t}$ as $\lambda_{1},\lambda_{2}\dots,\lambda_{d}$, then
\begin{align}
  \det(\uppercase{V}_{t})&=\lambda_{1}\times\lambda_{2}\times\dots\lambda_{d}\nonumber\\
               &\le\left(\displaystyle\frac{\lambda_{1}+\lambda_{2}+\dots+\lambda_{d}}{d}\right)^{d}\nonumber\\
               &=\left(\textrm{trace}\left(\uppercase{V}_{t}\right)/d\right)^d.\label{eq:lem11}
\end{align}

\noindent Since $V_{t}$ can be represented as
$V_{t} =\lambda I+\sum\limits_{s\in\uppercase{N}_{t}}\sum\limits_{e\in\uppercase{A}_{s}}{{x}_{s,e}{x}_{s,e}^\mathrm{T}}$, it gives that
\begin{align}
\textrm{trace}(V_{t})&=\textrm{trace}(\lambda I)+\sum\limits_{s\in\uppercase{N}_{t}}\sum\limits_{e\in\uppercase{A}_{s}}{\textrm{trace}({x}_{s,e}{x}_{s,e}^\mathrm{T})}\nonumber\\
&=\lambda d+\sum\limits_{s\in\uppercase{N}_{t}}\sum\limits_{e\in\uppercase{A}_{s}}{\Vert x_{s,e} \Vert_{2}^2}\nonumber\\
&\le \lambda d+ n_{t}KL.\label{eq:lem12}
\end{align}
Combining \eqref{eq:lem11} and \eqref{eq:lem12}, we obtain
\begin{align}
\det(V_{t}) \le (\lambda + n_{t}KL/d)^d.
\end{align}
\end{proof}

\section{\leftline{Proof of Lemma 3}} 
\begin{appenlem1}
\noindent Assume that $\Vert\theta_{*}\Vert_2\le\uppercase{S}$ and $\Vert{x}_{t,e}\Vert^{2}_2\le\uppercase{L}$, for any $t\ge 1$ and any $e\in\uppercase{A}_t$. If $\lambda\ge L$, then for any $\delta$,with probability at least $1-\delta$ and for any $T\ge 1$, the regret bound for selecting the optimistic policy is
\begin{align}
  &\sum\limits_{t\in\uppercase{N}_{T}}\left[f(A_{t},\uppercase{U}_{t})-f(A_{t},w^{*}_{t})\right]\nonumber\\
  &\le 2PC_{T}\sqrt{2dn_{T}\log\left(1+\displaystyle\frac{n_{T}KL}{d\lambda}\right)},\nonumber
\end{align}
where $C_{T}=\sqrt{\lambda}\uppercase{S}+\sqrt{2\log\left(\displaystyle\frac{1}{\delta}\right)+\lowercase{d}\log{\left(1+\displaystyle\frac{TKL}{\lambda d}\right)}}$.
\end{appenlem1}
\begin{proof}
According to the property claimed in \textbf{Lemma 1}, we know that $\uppercase{U}_{t}\ge w^{*}_{t}$, combined with the monotonicity of the reward function, we have $f(A^{*}_{t},w^{*}_{t})\le f(A^{*}_{t},\uppercase{U}_{t})$. Since $\uppercase{A}_{t}=\operatorname{argmax}_{\uppercase{A}\in\Theta}{f(A,\uppercase{U}_{t})}$, we also have $f(A^{*}_{t},\uppercase{U}_{t})\le f(A_{t},\uppercase{U}_{t})$.
The cumulative regret of selecting the optimistic policy can be bounded as follows,
\begin{align}
  &\sum\limits_{t\in\uppercase{N}_{T}}\left[f(A_{t},\uppercase{U}_{t})-f(A_{t},w^{*}_{t})\right]\nonumber\\
  &\le P\sum\limits_{t\in\uppercase{N}_{T}}\sqrt{\sum\limits_{e\in\uppercase{A}_{t}}{4C^{2}_{t}\Vert{x}_{t,e}\Vert^{2}_{V_{t-1}^{-1}}}}\nonumber\\
  &\le P\sqrt{\sum\limits_{t\in\uppercase{N}_{T}}\sum\limits_{e\in\uppercase{A}_{t}}{4n_{T}C^{2}_{t}\Vert{x}_{t,e}\Vert^{2}_{V_{t-1}^{-1}}}}\nonumber\\
  &\le 2PC_{T}\sqrt{\sum\limits_{t\in\uppercase{N}_{T}}\sum\limits_{e\in\uppercase{A}_{t}}{n_{T}\Vert{x}_{t,e}\Vert^{2}_{V_{t-1}^{-1}}}}.\nonumber
\end{align}
If $\lambda\ge L$, then
$\Vert{x}_{t,e}\Vert^{2}_{V_{t-1}^{-1}}\le\displaystyle\frac{\Vert{x}_{t,e}\Vert^{2}_{2}}{\lambda}\le\displaystyle\frac{L}{\lambda}\le1$.
Since $x\le2\log(1+x)$ for $x\in[0,1]$, we have
\begin{align}
&\sum\limits_{t\in\uppercase{N}_{T}}\sum\limits_{e\in\uppercase{A}_{t}}{\Vert{x}_{t,e}\Vert^{2}_{V_{t-1}^{-1}}}\nonumber\\
&\le\sum\limits_{t\in\uppercase{N}_{T}}2\log\left(1+\sum\limits_{e\in\uppercase{A}_{t}}{\Vert{x}_{t,e}\Vert^{2}_{V_{t-1}^{-1}}}\right).\label{eq:46}
\end{align}
Denote the rounds selecting the optimistic arm sets as $t_{1},t_{2},\cdots, t_{n_{T}}$. From the recursion formula of $V_{t}$, we have
\begin{align}
&\det( V_{t_{n_{T}}})=\det( V_{t_{n_{T}-1}}+\sum\limits_{e\in\uppercase{A}_{t_{n_{T}}}}{{x}_{t_{n_{T}},e}({x}_{t_{n_{T}},e})^\mathrm{T}})\nonumber\\
&=\det( V_{t_{n_{T}-1}})\nonumber\\
 &\quad\cdot det\bigg(I+\sum\limits_{e\in\uppercase{A}_{t_{n_{T}}}}{(V_{t_{n_{T}-1}}^{-1/2}{x}_{t_{n_{T}},e})(V_{t_{n_{T}-1}}^{-1/2}{x}_{t_{n_{T}},e})^\mathrm{T}}\bigg)\nonumber\\
&\ge\det( V_{t_{n_{T}-1}})\left(1+\sum\limits_{e\in\uppercase{A}_{t_{n_{T}}}}{\Vert{x}_{t_{n_{T}},e}\Vert^{2}_{V_{t_{n_{T}-1}}^{-1}}}\right)\label{eq:app22}\\
&\ge\det( V_{t_{1}-1}) \prod_{t\in\uppercase{N}_{T}}\left(1+\sum\limits_{e\in\uppercase{A}_{t}}{\Vert{x}_{t,e}\Vert^{2}_{V_{t-1}^{-1}}}\right),\nonumber
\end{align}
from \eqref{eq:app22} we know that $\det(V_{t}) \ge \det(V_{t-1})$ for any $t\ge1$, combined with the property that $1\le t_{1}\le t_{n_{T}}\le T$, we have
\begin{align}
&\prod_{t\in\uppercase{N}_{T}}\left(1+\sum\limits_{e\in\uppercase{A}_{t}}{\Vert{x}_{t,e}\Vert^{2}_{V_{t-1}^{-1}}}\right)\label{eq:47}
\le \det(V_{T})/\det(V_{0}).
\end{align}
Combine \eqref{eq:46} and \eqref{eq:47}, we have
\begin{align}
&\sum\limits_{t\in\uppercase{N}_{T}}\sum\limits_{e\in\uppercase{A}_{t}}{\Vert{x}_{t,e}\Vert^{2}_{V_{t-1}^{-1}}}
\le2\log(\det(V_{T}))-2\log(\det(\lambda I))\label{eq:31}.
\end{align}
In addition, \textbf{Lemma 2} informs us that
\begin{align}
\det(V_{T}) \le (\lambda d+ n_{T}KL/d)^d.\label{eq:32}
\end{align}
Combine \eqref{eq:31} and \eqref{eq:32} gives
\begin{align}
&\sum\limits_{t\in\uppercase{N}_{T}}\sum\limits_{e\in\uppercase{A}_{t}}{\Vert{x}_{t,e}\Vert^{2}_{V_{t-1}^{-1}}}
\le 2d\log\left(1+\displaystyle\frac{n_{T}KL}{d\lambda}\right).
\end{align}
Thus,
\begin{align}
  &\sum\limits_{t\in\uppercase{N}_{T}}\left[f(A_{t},\uppercase{U}_{t})-f(A_{t},w^{*}_{t})\right]\nonumber\\
  &\le 2PC_{T}\sqrt{2dn_{T}\log\left(1+\displaystyle\frac{n_{T}KL}{d\lambda}\right)}.\nonumber
\end{align}
\end{proof}

\section{\leftline{Proof of Lemma 6}}
\begin{appenlem2}
For any $t\in\uppercase{D}_{T}$, we have
\begin{align}
  &d_{t}<{\left(\left[1-(1+n_{t})\alpha\right]\mu_0-n_{t}\Delta_{min}\right)/(\alpha\mu_0)}\nonumber\\
  &+{2PC_{t}\sqrt{n_{t}\sum\limits_{n\in\uppercase{N}_{t}}\sum\limits_{e\in\uppercase{A}_{n}}{\Vert{x}_{n,e}\Vert^{2}_{V_{t}^{-1}}}}/(\alpha\mu_0)},\nonumber
\end{align}
where $C_{t}=\sqrt{\lambda}\uppercase{S}+\sqrt{2\log\left(1/\delta\right)+\lowercase{d}\log{\left(1+\displaystyle\frac{KLt}{\lambda d}\right)}}$.
\end{appenlem2}
\begin{proof}
Suppose the conservative arm set $A_{0}$ is selected at round $t$, then according to \textbf{Algorithm 1}, it is satisfied that
\begin{align}
  \sum_{n\in\uppercase{N}_{t-1}}f(A_{n},L_{t,n})&+f(B_{t},L_{t,t})\nonumber\\
        &+d_{t-1}\mu_{0}<(1-\alpha)t\mu_0\label{eq:aa}
\end{align}

\noindent Note that $t$ can be denoted as $t=n_{t-1}+d_{t-1}+1$. By dropping $f(B_{t},L_{t,t})$, and rearranging the terms in \eqref{eq:aa}, we have
\begin{align}
  &\alpha d_{t-1}\mu_0\nonumber\\
  &<\left[1-(1+n_{t-1})\alpha\right]\mu_0+n_{t-1}\mu_0-\sum\limits_{n\in\uppercase{N}_{t-1}}f(A_{n},L_{t,n})\nonumber\\
  &=\left[1-(1+n_{t-1})\alpha\right]\mu_0+\sum\limits_{n\in\uppercase{N}_{t-1}}\Big[\mu_0-f(A_{n},U_{t,n})\nonumber\\
  &\quad+f(A_{n},U_{t,n})-f(A_{n},L_{t,n})\Big]\nonumber\\
  &\le\left[1-(1+n_{t-1})\alpha\right]\mu_0+\sum\limits_{n\in\uppercase{N}_{t-1}}\Big[\mu_{0}-f(A_{n},U_{t,n})\nonumber\\
  &\quad+P\sqrt{\sum\limits_{e\in\uppercase{A}_{n}}\left(U_{t,n,e}-L_{t,n,e}\right)^2}\Big]\label{eq:le51}\\
  &=\left[1-(1+n_{t-1})\alpha\right]\mu_0+\sum\limits_{n\in\uppercase{N}_{t-1}}\Big[\mu_{0}-f(A_{n},U_{t,n})\nonumber\\
  &\quad+2PC_{t-1}\sqrt{\sum\limits_{e\in\uppercase{A}_{n}}{\Vert{x}_{n,e}\Vert_{V_{t-1}^{-1}}^2}}\Big]\label{eq:le53}\\
  &\le\left[1-(1+n_{t-1})\alpha\right]\mu_0-n_{t-1}\Delta_{min}\nonumber\\
  &\quad+2PC_{t-1}\sqrt{n_{t-1}\sum\limits_{n\in\uppercase{N}_{t-1}}\sum\limits_{e\in\uppercase{A}_{n}}{\Vert{x}_{n,e}\Vert^{2}_{V_{t-1}^{-1}}}}\label{eq:le52},
\end{align}
where the \eqref{eq:le51} is due to the lipschitz continuity of the reward function, \eqref{eq:le53} follows from \textbf{Lemma 1}, and \eqref{eq:le52} follows from Cauchy-Schwarz inequality.
\end{proof}

\section{\leftline{Proof of Theorem 1}}
\begin{appenthm1}
\noindent If $\lambda\ge L$, then the following regret bound is satisfied with probability at least $1 - \delta$
\begin{align}
  R(T)< &2PC_{T}\sqrt{2dT\log\left(1+TKL/\left(\lambda d\right)\right)}\nonumber\\
  &\quad+\Big[1/\alpha+\left(\alpha\mu_{0}+\Delta_{min}\right)\lambda d/\left(\alpha\mu_{0}KL\right)\nonumber\\
  &\quad+4P^{2}C^{2}_{T-1}d/(3\alpha\mu_{0}\left(\alpha\mu_{0}+\Delta_{min}\right))\Big]\Delta_{max},\nonumber\\
  =&
  O\left(d(d+\sqrt{T})\max\{1,\log(\displaystyle\frac{TK}{d})\}+\displaystyle\frac{d}{K}\right),\nonumber
\end{align}
where
$C_{T}=\sqrt{\lambda}\uppercase{S}+\sqrt{2\log\left(1/\delta\right)+\lowercase{d}\log{\left(1+\displaystyle\frac{TKL}{\lambda d}\right)}}\nonumber$.
Furthermore, for $TK\le d$, we have
\begin{align}
  R(T)<&2PC_{T}\sqrt{2dT\log\left(1+TKL/\left(\lambda d\right)\right)}\nonumber\\
  &\quad+\Big[\displaystyle\frac{1}{\alpha}+\displaystyle\frac{2Pd}{\alpha\mu_{0}}\sqrt{\displaystyle\frac{L}{K\left(\lambda+L\right)}}C_{T-1}\Big]\Delta_{max},\nonumber\\
  =&
  O\left(d\sqrt T+d^{3/2}/\sqrt K\right).\nonumber
\end{align}
\end{appenthm1}
\begin{proof}
Suppose $t^{'}$ is the last round the conservative policy is played, then
\begin{align}
  &d_{t^{'}-1}=d_{T}-1,
\end{align}
and by \textbf{Lemma 6} we have
\begin{align}
&\alpha d_{t^{'}-1}\mu_0\nonumber\\
&<\left[1-(1+n_{t^{'}-1})\alpha\right]\mu_0-n_{t^{'}-1}\Delta_{min}\nonumber\\
  &\quad+2PC_{t^{'}-1}\sqrt{n_{t^{'}-1}\sum\limits_{n\in\uppercase{N}_{t^{'}-1}}\sum\limits_{e\in\uppercase{A}_{n}}{\Vert{x}_{n,e}\Vert^{2}_{V_{t^{'}-1}^{-1}}}}\nonumber\\
&<\left[1-(1+n_{t^{'}-1})\alpha\right]\mu_0-n_{t^{'}-1}\Delta_{min}\nonumber\\
  &\quad+2PC_{t^{'}-1}n_{t^{'}-1}\sqrt{\displaystyle\frac{KLd}{\lambda d + n_{t^{'}-1}KL}},\label{eq:ae}
\end{align}
where \eqref{eq:ae} follows from \textbf{Lemma 5}.

\noindent Define
\begin{align}
 f(x)=-\left(\alpha\mu_{0}+\Delta_{min}\right)x+2PC_{T-1}x\sqrt{\displaystyle\frac{KLd}{\lambda d+KLx}},\nonumber
\end{align}
denote
$\left(\alpha\mu_{0}+\Delta_{min}\right)/KL=A$, $2PC_{T-1}\sqrt{\displaystyle\frac{d}{KL}}=B$,
then
\begin{align}
f(x)=-AKLx+\displaystyle\frac{BKLx}{\sqrt{KLx+\lambda d}},\nonumber
\end{align}
with the substitution $u=KLx+\lambda d$, we have
\begin{align}
f(u)=-Au+A\lambda d+B\sqrt{u}-\displaystyle\frac{B\lambda d}{\sqrt{u}}\label{eq:ac}.
\end{align}
Note that
\begin{align}
 f^{''}(u)=-\displaystyle\frac{B}{4}u^{-3/2}-\displaystyle\frac{3B\lambda d}{4}u^{-5/2}<0,\nonumber
\end{align}
which signifies that the global maximum of $\lowercase{f}$ could be reached at the unique stationary point.
Set the first order $f^{'}\left(t\right)=0$ gives
\begin{align}
\displaystyle\frac{B\lambda d}{\sqrt{u}}=2At-B\sqrt{u}\label{eq:ah}.
\end{align}
Substitute \eqref{eq:ah} into \eqref{eq:ac} gives
\begin{align}
f(u)=-3At+2B\sqrt{u}+A\lambda d
\le\uppercase{A}\lambda d+\displaystyle\frac{B^{2}}{3A}.\nonumber
\end{align}
Thus,
\begin{align}
 &d_{T}=d_{t^{'}-1}+1\nonumber\\
 &<1/\alpha+\left(\alpha\mu_{0}+\Delta_{min}\right)\lambda d/\left(\alpha\mu_{0}KL\right)\nonumber\\
 &\quad+\displaystyle\frac{4P^{2}C^{2}_{T-1}d}{3\left(\alpha\mu_{0}+\Delta_{min}\right)\alpha\mu_{0}}\label{eq:th121}.
\end{align}
According to the property claimed in \textbf{Lemma 1}, we know that $\uppercase{U}_{t}\ge w^{*}_{t}$, combined with the monotonicity of the reward function, we have $f(A^{*}_{t},w^{*}_{t})\le f(A^{*}_{t},\uppercase{U}_{t})$. Since $\uppercase{A}_{t}=\operatorname{argmax}_{\uppercase{A}\in\Theta}{f(A,\uppercase{U}_{t})}$, we also have $f(A^{*}_{t},\uppercase{U}_{t})\le f(A_{t},\uppercase{U}_{t})$.

\noindent After running \textbf{Algorithm 1} for $T$ rounds, the cumulative regret can be bounded as follows,
\begin{align}
  R(T)&=\sum\limits_{t=1}^{T}\left[f(A^{*}_{t},w^{*}_{t})-f(A_{t},w^{*}_{t})\right]\nonumber\\
  &=\sum\limits_{t\in\uppercase{N}_{T}}\left[f(A^{*}_{t},w^{*}_{t})-f(A_{t},w^{*}_{t})\right]\nonumber\\
  &\quad+\sum\limits_{t\in\uppercase{D}_{T}}\left[f(A^{*}_{t},w^{*}_{t})-\mu_0\right]\nonumber\\
  &\le\sum\limits_{t\in\uppercase{N}_{T}}\left[f(A_{t},\uppercase{U}_{t})-f(A_{t},w^{*}_{t})\right]\nonumber\\
  &\quad+\sum\limits_{t\in\uppercase{D}_{T}}\left[f(A^{*}_{t},w^{*}_{t})-\mu_0\right]\\
  &\le\sum\limits_{t\in\uppercase{N}_{T}}\left[f(A_{t},\uppercase{U}_{t})-f(A_{t},w^{*}_{t})\right]+d_{T}\Delta_{max}\label{eq:th112}\\
  &< 2PC_{T}\sqrt{2dT\log\left(1+\displaystyle\frac{TKL}{d\lambda}\right)}\nonumber\\
  &\quad+\Big[1/\alpha+ \left(\alpha\mu_{0}+\Delta_{min}\right)\lambda d/\left(\alpha\mu_{0}KL\right)\nonumber\\
  &\quad+\displaystyle\frac{4P^{2}C^{2}_{T-1}d}{3\left(\alpha\mu_{0}+\Delta_{min}\right)\alpha\mu_{0}}\Big]\Delta_{max}\label{eq:th113},
\end{align}
where the first part of \eqref{eq:th113} follows from \textbf{Lemma 3} and the fact that $n_T\le T$, and the second part of \eqref{eq:th113} follows from \eqref{eq:th121}.

\noindent Furthermore, if $TK\le d$, then it implies that $n_{t^{'}}K\le d$, we apply \textbf{Lemma 4} and have
\begin{align}
&\alpha d_{t^{'}-1}\mu_0\nonumber\\
&<\left[1-(1+n_{t^{'}-1})\alpha\right]\mu_0-n_{t^{'}-1}\Delta_{min}+\nonumber\\
&2PC_{t^{'}-1}n_{t^{'}-1}\sqrt{K\left[1-(\displaystyle\frac{\lambda}{\lambda+n_{t^{'}-1}KL/d})^{d/\left(n_{t^{'}-1}K\right)}\right]}\nonumber\\
&\le\left(1-\alpha\right)\mu_{0}+2Pd\sqrt{\displaystyle\frac{L}{K\left(\lambda+L\right)}}C_{t^{'}-1}.\nonumber
\end{align}

\noindent Thus,
\begin{align}
 d_{T}&=d_{t^{'}-1}+1\nonumber\\
 &<\displaystyle\frac{1}{\alpha}+\displaystyle\frac{2Pd}{\alpha\mu_{0}}\sqrt{\displaystyle\frac{L}{K\left(\lambda+L\right)}}C_{T-1}\label{eq:th1101}.
\end{align}

\noindent Combine \eqref{eq:th1101} with \eqref{eq:th112} we have
\begin{align}
  R(T)< &2PC_{T}\sqrt{2dT\log\left(1+\displaystyle\frac{TKL}{d\lambda}\right)}\nonumber\\
  &\quad+\Big[1/\alpha+\displaystyle\frac{2Pd}{\alpha\mu_{0}}\sqrt{\displaystyle\frac{L}{K\left(\lambda+L\right)}}C_{T-1}\Big]\Delta_{max}.\nonumber
\end{align}
\end{proof}

\section{\leftline{Proof of Theorem 2}}
\begin{appenthm}
\noindent If $\lambda\ge L$, then the following regret bound is satisfied with probability at least $1 - \delta$
\begin{align}
  R(T)< &2PC_{T}\sqrt{2dT\log\left(1+TKL/\left(d\lambda\right)\right)}\nonumber\\
  &+max\{\Big[1/\alpha+2\left(1-\alpha\right)PC_{T-1}/\left(\alpha\mu_{0}\right)\sqrt{KL/\lambda}\nonumber\\
  &+\left(\displaystyle\frac{\alpha\mu_{0}+\Delta_{min}}{\alpha\mu_{0}KL}+
  \displaystyle\frac{2\alpha PC_{T-1}}{\alpha\mu_{0}\sqrt{KL\lambda}}\right)\lambda d\nonumber\\
 &+\displaystyle\frac{4P^{2}C^{2}_{T-1}d}{3\left(\alpha\mu_{0}+\Delta_{min}+2\alpha P\sqrt{KL/\lambda}C_{T-1}\right)\alpha\mu_{0}}\Big],\nonumber\\
 &\Big[1/\alpha+\left(\alpha\mu_{0}+\Delta_{min}\right)\lambda d/\left(\alpha\mu_{0}KL\right)\nonumber\\
  &+\displaystyle\frac{4P^{2}C^{2}_{T-1}d}{3\left(\alpha\mu_{0}+\Delta_{min}\right)\alpha\mu_{0}}\Big]\}\Delta_{max},\nonumber\\
  =&\, O(d\sqrt{T} \max\{1,\log(TK/d)\}\nonumber\\
  &+\max\{d^{2}\max\{1,\log(TK/d)\}, \nonumber\\
  &\sqrt{dK\max\{1,\log(TK/d)\}}\}),\nonumber
\end{align}
where $C_{T}=\sqrt{2\log\left(\displaystyle\frac{1}{\delta}\right)+\lowercase{d}\log{\left(1+\displaystyle\frac{TKL}{\lambda d}\right)}}+\sqrt{\lambda}\uppercase{S}$.
Furthermore, if $TK\le d$, then
\begin{align}
  &R(T)< 2PC_{T}\sqrt{2dT\log\left(1+TKL/\left(d\lambda\right)\right)}\nonumber\\
  &+\Big[1/\alpha+2P\left(1-\alpha\right)\sqrt{KL/\lambda}C_{T-1}/\left(\alpha\mu_{0}\right)\nonumber\\
  &+2Pd/\left(\alpha\mu_{0}\right)\sqrt{L/\left(K\left(\lambda+L\right)\right)}C_{T-1}\Big]\Delta_{max},\nonumber\\
  &=
  O(d\sqrt T+d^{3/2}/\sqrt{K}+\sqrt{Kd}).\nonumber
\end{align}
\end{appenthm}

\begin{proof}
If we want to choose the optimistic arm set at round t, then the cumulative reward should not be less than certain percent of the reward gained by choosing the conservative arm set. That is,
\begin{align}
  &\sum_{n\in\uppercase{N}_{t-1}}f(A_{n},w_{n}^{*})+f(B_{t},w_{t}^{*})\nonumber\\
  &\quad+\sum_{d\in\uppercase{D}_{t-1}}f(A_{0},w_{0}^{*})-(1-\alpha)\sum\limits_{t=1}^{t}f(A_{0},w_{0}^{*})\label{eq:afd}
\end{align}
should be non-negative.

\noindent To make the selection at each round less risky, the conservative learner wants to ensure that the calculable lower bound of \eqref{eq:afd} is non-negative.

\noindent Consider that $f(A_{0},U_{t,0})\ge f(A_{0},w_{0}^{*})$, then
\begin{align}
  \sum_{n\in\uppercase{N}_{t-1}}&f(A_{n},L_{t,n})+f(B_{t},L_{t,t})\nonumber\\
  &+\left(d_{t-1}-\left(1-\alpha\right)t\right)f(A_{0},U_{t,0})\nonumber
\end{align}
is a lower bound of \eqref{eq:afd} when $d_{t-1}< (1-\alpha)t$. Suppose $t^{'}$ is the last round the conservative policy is played, then
\begin{align}
  &d_{t^{'}-1}=d_{T}-1,
\end{align}
and
\begin{align}
  &\sum_{n\in\uppercase{N}_{t^{'}-1}}f(A_{n},L_{t^{'},n})+f(B_{t^{'}},L_{t^{'},t^{'}})\nonumber\\
        &\quad+d_{t^{'}-1}f(A_{0},U_{t^{'}, 0})<(1-\alpha)t^{'}f(A_{0},U_{t^{'}, 0}).\label{eq:th21}
\end{align}
By \textbf{Lemma 6}, we have
\begin{align}
&\alpha d_{t^{'}-1}f(A_{0},U_{t^{'}, 0})\nonumber\\
&<\left[1-(1+n_{t^{'}-1})\alpha\right]f(A_{0},U_{t^{'}, 0})-n_{t^{'}-1}\Delta_{min}\nonumber\\
  &\quad+2PC_{t^{'}-1}\sqrt{n_{t^{'}-1}\sum\limits_{n\in\uppercase{N}_{t^{'}-1}}\sum\limits_{e\in\uppercase{A}_{n}}{\Vert{x}_{n,e}\Vert^{2}_{V_{t^{'}-1}^{-1}}}}.\label{eq:200}
\end{align}
From the lipschitz continuity of the expected reward function and \textbf{Lemma 1}, we have
\begin{align}
f(A_{0},U_{t^{'}, 0})-f(A_{0},w^{*}_{0})
\le 2PC_{t^{'}-1}\sqrt{\sum\limits_{e\in\uppercase{A}_{0}}{\Vert{x}_{0,e}\Vert^{2}_{V_{t^{'}-1}^{-1}}}}\label{eq:th22}.
\end{align}
Since $\Vert{x}_{0,e}\Vert_{V_{t^{'}-1}^{-1}}\le\displaystyle\frac{\Vert{x}_{0,e}\Vert_{2}}{\sqrt{\lambda}}\le\sqrt{\displaystyle\frac{L}{\lambda}}$, combined with \eqref{eq:th22}, we have
\begin{align}
f(A_{0},U_{t^{'}, 0})&\le f(A_{0},w^{*}_{0})+2P\sqrt{\displaystyle\frac{|A_{0}|L}{\lambda}}C_{t^{'}-1}\nonumber\\
&\le \mu_{0}+2P\sqrt{\displaystyle\frac{KL}{\lambda}}C_{t^{'}-1}.\label{eq:201}
\end{align}
If $1-(1+n_{t^{'}-1})\alpha<0$, then from \eqref{eq:200} and \eqref{eq:201} we have
\begin{align}
&\alpha d_{t^{'}-1}\mu_{0}\nonumber\\
&<\left[1-(1+n_{t^{'}-1})\alpha\right]\mu_{0}-n_{t^{'}-1}\Delta_{min}\nonumber\\
  &\quad+2PC_{t^{'}-1}\sqrt{n_{t^{'}-1}\sum\limits_{n\in\uppercase{N}_{t^{'}-1}}\sum\limits_{e\in\uppercase{A}_{n}}{\Vert{x}_{n,e}\Vert^{2}_{V_{t^{'}-1}^{-1}}}},\nonumber\\
\end{align}
which is consistent with the case when $\mu_{0}$ is known. Thus, we have
\begin{align}
d_{T}<&1/\alpha+ \left(\alpha\mu_{0}+\Delta_{min}\right)\lambda d/\left(\alpha\mu_{0}KL\right)\nonumber\\
  &+\displaystyle\frac{4P^{2}C^{2}_{T-1}d}{3\left(\alpha\mu_{0}+\Delta_{min}\right)\alpha\mu_{0}},\nonumber
\end{align}
and
\begin{align}
R(T)< &2PC_{T}\sqrt{2dT\log\left(1+\displaystyle\frac{TKL}{d\lambda}\right)}\nonumber\\
  &\quad+\Big[1/\alpha+\left(\alpha\mu_{0}+\Delta_{min}\right)\lambda d/\left(\alpha\mu_{0}KL\right)\nonumber\\
  &\quad+\displaystyle\frac{4P^{2}C^{2}_{T-1}d}{3\left(\alpha\mu_{0}+\Delta_{min}\right)\alpha\mu_{0}}\Big]\Delta_{max}.\nonumber
\end{align}

If $1-(1+n_{t^{'}-1})\alpha\ge0$, then from \eqref{eq:200} and \eqref{eq:201} we have
\begin{align}
&\alpha d_{t^{'}-1}\mu_0\nonumber\\
&<\left(1-\alpha\right)\left(\mu_{0}+2P\sqrt{\displaystyle\frac{KL}{\lambda}}C_{T-1}\right)\nonumber\\
&\quad-n_{t^{'}-1}\left(\alpha\mu_{0}+\Delta_{min}+2\alpha P\sqrt{\displaystyle\frac{KL}{\lambda}}C_{T-1}\right)\nonumber\\
  &\quad+2PC_{t^{'}-1}\sqrt{n_{t^{'}-1}\sum\limits_{n\in\uppercase{N}_{t^{'}-1}}\sum\limits_{e\in\uppercase{A}_{n}}{\Vert{x}_{n,e}\Vert^{2}_{V_{t^{'}-1}^{-1}}}}\nonumber\\
&\le\left(1-\alpha\right)\left(\mu_{0}+2P\sqrt{\displaystyle\frac{KL}{\lambda}}C_{T-1}\right)\nonumber\\
&\quad-n_{t^{'}-1}\left(\alpha\mu_{0}+\Delta_{min}+2\alpha P\sqrt{\displaystyle\frac{KL}{\lambda}}C_{T-1}\right)\nonumber\\
  &\quad+2PC_{T-1}n_{t^{'}-1}\sqrt{\displaystyle\frac{KLd}{\lambda d + n_{t^{'}-1}KL}},\nonumber
\end{align}
where the second inequality follows from \textbf{Lemma 5}.

\noindent Denote by $f(x)=-AKLx+\displaystyle\frac{BKLx}{\sqrt{KLx+\lambda d}}$, where
$A=\displaystyle\frac{\alpha\mu_{0}+\Delta_{min}}{KL}+\displaystyle\frac{2\alpha P}{\sqrt{KL\lambda}}C_{T-1}$,
$B=2PC_{T-1}\sqrt{\displaystyle\frac{d}{KL}}$.
According to the proof of \textbf{Theorem 1}, we have $f(x)\le\uppercase{A}\lambda d+\displaystyle\frac{B^{2}}{3A}$.
Thus,
\begin{align}
 d_{T}&=d_{t^{'}-1}+1\nonumber\\
 &<1/\alpha+\displaystyle\frac{2\left(1-\alpha\right)PC_{T-1}}{\alpha\mu_{0}}\sqrt{\displaystyle\frac{KL}{\lambda}}\nonumber\\
 &\quad+\Big[\displaystyle\frac{\alpha\mu_{0}+\Delta_{min}}{\alpha\mu_{0}KL}+
 \displaystyle\frac{2\alpha P}{\alpha\mu_{0}\sqrt{KL\lambda}}C_{T-1}\Big]\lambda d\nonumber\\
 &\quad+\displaystyle\frac{4P^{2}C^{2}_{T-1}d}{3\left(\alpha\mu_{0}+\Delta_{min}+2\alpha P\sqrt{KL/\lambda}C_{T-1}\right)\alpha\mu_{0}}\label{eq:th221}.
\end{align}
After running \textbf{Algorithm 2} for $T$ rounds, the cumulative regret could be bounded as follows,
\begin{align}
  R(T)&=\sum\limits_{t=1}^{T}\left[f(A^{*}_{t},w^{*}_{t})-f(A_{t},w^{*}_{t})\right]\nonumber\\
  &\le\sum\limits_{t\in\uppercase{N}_{T}}\left[f(A_{t},\uppercase{U}_{t})-f(A_{t},w^{*}_{t})\right]+d_{T}\Delta_{max}\label{eq:th216}\\
  &< 2PC_{T}\sqrt{2dT\log\left(1+\displaystyle\frac{TKL}{d\lambda}\right)}\nonumber\\
  &\quad+\Big[1/\alpha+\displaystyle\frac{2\left(1-\alpha\right)PC_{T-1}}{\alpha\mu_{0}}\sqrt{\displaystyle\frac{KL}{\lambda}}\nonumber\\
 &\quad+\left(\displaystyle\frac{\alpha\mu_{0}+\Delta_{min}}{\alpha\mu_{0}KL}+
 \displaystyle\frac{2\alpha P}{\alpha\mu_{0}\sqrt{KL\lambda}}C_{T-1}\right)\lambda d\nonumber\\
 &\quad+\displaystyle\frac{4P^{2}C^{2}_{T-1}d}{3\left(\alpha\mu_{0}+\Delta_{min}+2\alpha P\sqrt{KL/\lambda}C_{T-1}\right)\alpha\mu_{0}}\Big]\Delta_{max},\label{eq:th217}
\end{align}
where \eqref{eq:th216} follows from the proof in \textbf{Theorem 1}, \eqref{eq:th217} follows from \textbf{Lemma 3} and \eqref{eq:th221}.

\noindent Furthermore, if $TK\le d$, then it implies that $n_{t^{'}}K\le d$.

\noindent If $1-(1+n_{t^{'}-1})\alpha<0$, the analysis is consistent with the case when $\mu_{0}$ is known.
Thus,
\begin{align}
 &d_{T}<\displaystyle\frac{1}{\alpha}+\displaystyle\frac{2Pd}{\alpha\mu_{0}}\sqrt{\displaystyle\frac{L}{K\left(\lambda+L\right)}}C_{T-1}.\label{eq:268}
\end{align}
If $1-(1+n_{t^{'}-1})\alpha\ge 0$, we apply \textbf{Lemma 4} and have
\begin{align}
 &\alpha d_{t^{'}-1}\mu_0\nonumber\\
&<\left(1-\alpha\right)\left(\mu_{0}+2P\sqrt{\displaystyle\frac{KL}{\lambda}}C_{T-1}\right)\nonumber\\
&\quad-n_{t^{'}-1}\left(\alpha\mu_{0}+\Delta_{min}+2\alpha P\sqrt{\displaystyle\frac{KL}{\lambda}}C_{T-1}\right)\nonumber\\
&\quad+2PC_{t^{'}-1}n_{t^{'}-1}\sqrt{K\left[1-(\displaystyle\frac{\lambda}{\lambda+n_{t^{'}-1}KL/d})^{d/(n_{t^{'}-1}K)}\right]}\nonumber\\
&\le\left(1-\alpha\right)\left(\mu_{0}+2P\sqrt{\displaystyle\frac{KL}{\lambda}}C_{T-1}\right)\nonumber\\
&\quad+2Pd\sqrt{\displaystyle\frac{L}{K\left(\lambda+L\right)}}C_{t^{'}-1}\nonumber.
\end{align}

\noindent Thus,
\begin{align}
 d_{T}=&d_{t^{'}-1}+1\nonumber\\
 <&1/\alpha+2P\left(1-\alpha\right)\sqrt{\displaystyle\frac{KL}{\lambda}}C_{T-1}/\left(\alpha\mu_{0}\right)\nonumber\\
  &\quad+\displaystyle\frac{2Pd}{\alpha\mu_{0}}\sqrt{\displaystyle\frac{L}{K\left(\lambda+L\right)}}C_{T-1}\label{eq:th2101}.
\end{align}

\noindent Combine \eqref{eq:th216}, \eqref{eq:268} and \eqref{eq:th2101} we have

\begin{align}
  R(T)< &2PC_{T}\sqrt{2dT\log\left(1+\displaystyle\frac{TKL}{d\lambda}\right)}\nonumber\\
  &\quad+\Big[1/\alpha+2P\left(1-\alpha\right)\sqrt{\displaystyle\frac{KL}{\lambda}}C_{T-1}/\left(\alpha\mu_{0}\right)\nonumber\\
  &\quad+\displaystyle\frac{2Pd}{\alpha\mu_{0}}\sqrt{\displaystyle\frac{L}{K\left(\lambda+L\right)}}C_{T-1}\Big]\Delta_{max}.\nonumber
\end{align}
\end{proof}

\end{appendix}


\end{document}